\documentclass[11pt]{article}

\usepackage{times}

\def\colorful{0}

\usepackage{amsthm,amsfonts,amsmath,amssymb,epsfig,color,float,graphicx,verbatim, enumitem}
\usepackage{multirow}

\newif\ifhyper\IfFileExists{hyperref.sty}{\hypertrue}{\hyperfalse}
\hypertrue
\ifhyper\usepackage{hyperref}\fi

\usepackage{enumitem}

\makeatletter
\renewcommand{\section}{\@startsection{section}{1}{0pt}{-12pt}{5pt}{\large\bf}}
\renewcommand{\subsection}{\@startsection{subsection}{2}{0pt}{-12pt}{-5pt}{\normalsize\bf}}
\renewcommand{\subsubsection}{\@startsection{subsubsection}{3}{0pt}{-12pt}{-5pt}{\normalsize\bf}}
\makeatother

\usepackage{framed}
\usepackage{nicefrac}

\def\nnewcolor{1}
\ifnum\nnewcolor=1

\fi
\ifnum\nnewcolor=0

\fi

\ifnum\colorful=1
\newcommand{\new}[1]{{\color{red} #1}}

\else
\newcommand{\new}[1]{{#1}}

\fi

\newtheorem{theorem}{Theorem}[section]

\newtheorem{lemma}[theorem]{Lemma}
\newtheorem{informal theorem}[theorem]{Theorem (informal statement)}

\newtheorem{proposition}[theorem]{Proposition}
\newtheorem{corollary}[theorem]{Corollary}
\newtheorem{claim}[theorem]{Claim}
\newtheorem{fact}[theorem]{Fact}

\newtheorem{remark}[theorem]{Remark}

\theoremstyle{definition}
\newtheorem{definition}[theorem]{Definition}
\newcommand{\eqdef}{\stackrel{{\mathrm {\footnotesize def}}}{=}}

\newcommand{\R}{\mathbb{R}}

\newcommand{\Z}{\mathbb{Z}}
\newcommand{\N}{\mathbb{N}}
\newcommand{\E}{\mathbf{E}}
\newcommand{\eps}{\epsilon}
\newcommand{\dtv}{d_{\mathrm TV}}

\newcommand{\poly}{\mathrm{poly}}
\newcommand{\var}{\mathbf{Var}}
\newcommand{\D}{D}

\newcommand{\sgn}{\mathrm{sign}}
\newcommand{\sign}{\mathrm{sign}}

\newcommand{\ba}{\mathbf{a}}

\title{Learning Geometric Concepts with Nasty Noise}

\author{
Ilias Diakonikolas\thanks{Supported by NSF Award CCF-1652862 (CAREER) and a Sloan Research Fellowship.}\\
University of Southern California\\
{\tt diakonik@usc.edu}\\
\and
Daniel M. Kane\thanks{Supported by NSF Award CCF-1553288 (CAREER) and a Sloan Research Fellowship.}\\
University of California, San Diego\\
{\tt dakane@cs.ucsd.edu}\\
\and
Alistair Stewart\\ University of Southern California\\
{\tt alistais@usc.edu}
}

\begin{document}

\maketitle

\thispagestyle{empty}

\begin{abstract}
We study the efficient learnability of geometric concept classes --  
specifically, low-degree polynomial threshold functions (PTFs)
and intersections of halfspaces -- when a fraction of the 
training data is adversarially corrupted. We give the first polynomial-time 
PAC learning algorithms for these concept classes with {\em dimension-independent} error guarantees
in the presence of {\em nasty noise} under the Gaussian distribution. In the nasty noise model, 
an omniscient adversary can arbitrarily corrupt a small fraction of both the unlabeled data points and their labels.
This model generalizes well-studied noise models,
including the malicious noise model and the agnostic (adversarial label noise) model. 
Prior to our work, the only concept class for which efficient malicious learning
algorithms were known was the class of {\em origin-centered} halfspaces~\cite{KLS:09jmlr, ABL17}.

Specifically, our robust learning algorithm for low-degree PTFs 
succeeds under a number of tame distributions -- including the Gaussian distribution 
and, more generally, any log-concave distribution with (approximately) known low-degree moments. 
For LTFs under the Gaussian distribution, we give 
a polynomial-time algorithm that achieves error $O(\eps)$, where $\eps$ is the noise rate. 
At the core of our PAC learning results is an efficient algorithm 
to approximate the {\em low-degree Chow-parameters}
of any bounded function in the presence of nasty noise.
To achieve this, we employ an iterative spectral method for outlier detection and removal, 
inspired by recent work in robust unsupervised learning. 
Our aforementioned algorithm succeeds for a range of distributions satisfying
mild concentration bounds and moment assumptions.
The correctness of our robust learning algorithm for intersections of halfspaces 
makes essential use of a novel robust inverse independence lemma 
that may be of broader interest.
\end{abstract}

\thispagestyle{empty}
\setcounter{page}{0}

\newpage

\section{Introduction}

\subsection{Background and Motivation} \label{ssec:motivation}
One of the central challenges in machine learning is to make accurate inferences from datasets
in which pieces of information are corrupted by noise. In this work, we study the efficient
learnability of geometric concepts --  
specifically, low-degree polynomial threshold functions 
and intersections of linear threshold functions (halfspaces) -- when a fraction of the 
training data is adversarially corrupted.
As our main contribution, we give the first polynomial-time PAC learning algorithms for these concept classes
with {\em dimension-independent} error guarantees\footnote{By the term ``dimension-independent'' error guarantee 
it is meant that: when the fraction of corruptions is $\eps$, our algorithms 
achieve error $f(\eps)$ (for some function $f:\R_+ \to \R_+$ such that $\lim_{x \to 0} f(x) =0$).}, 
in the presence of {\em nasty noise}.

Polynomial Threshold Functions (PTFs) and intersections of Linear Threshold Functions (LTFs)
are two fundamental classes of Boolean functions that have been extensively
studied in many contexts for at least the past five decades~\cite{Dertouzos:65, MinskyPapert:68, Muroga:71}.
In the noiseless setting, low-degree PTFs are known to be efficiently PAC learnable 
under arbitrary distributions via linear programming~\cite{MT:94}. 
The current state-of-the-art for PAC learning intersections of LTFs is as follows:
Even without noise, distribution-independent PAC learning 
for intersections of $2$ LTFs is \new{one of the most challenging open problems in computational learning theory}. 
Efficient algorithms are known for PAC learning intersections of any constant number of LTFs 
under well-behaved distributions, e.g., under the standard Gaussian distribution~\cite{Vempala10, Vempala10a}.
Dealing with (adversarial) noisy data turns out to be significantly more challenging in general. 
Recent results (see, e.g.,~\cite{DanielyLS14, Daniely16}) provide strong evidence 
that learning with adversarial noise is computationally intractable under arbitrary distributions, 
even for simple concept classes.

In this paper, we focus on the efficient learnability of low-degree PTFs
and intersections of (any constant number of) LTFs in the presence of {\em nasty noise}.
In the {\em nasty noise model}~\cite{BEK:02}, an omniscient adversary can 
arbitrarily corrupt a small fraction of both the unlabeled data points and their labels.
The nasty model generalizes a number of well-studied noise models,
including the malicious noise model~\cite{Valiant:85, KearnsLi:93}\footnote{In the malicious model, an adversary
can corrupt a small fraction of both the unlabeled examples and their labels. This model is qualitatively
similar to (but somewhat weaker than) the nasty noise model. We define these models and 
explain the relation between them in Section~\ref{ssec:noise-models}.}
and the agnostic (adversarial label noise) model~\cite{Haussler:92, KSS:94}. 
While these noise models were originally defined with respect to arbitrary 
distributions, it has been recently shown~\cite{Daniely16} (modulo plausible complexity assumptions) 
that, even for the class of LTFs, no computationally efficient algorithm can achieve dimension-independent
error guarantees. Hence, research in this area has focused on noise-tolerant learning 
under a number of ``tame'' distributions. Our goal in this paper is to design polynomial-time 
robust learning algorithms that can tolerate nasty noise of {\em constant} rate, i.e., 
we want to achieve error guarantees that are {\em independent of the dimension}. 

In the agnostic (adversarial label noise) PAC model~\cite{Haussler:92, KSS:94}, 
the $L_1$-regression algorithm of Kalai {\em et al.}~\cite{KKMS:08} 
can be used to learn low-degree PTFs and intersections of LTFs 
under a number of well-behaved distributions, 
including the uniform distribution on the hypercube and the standard Gaussian distribution;
see, e.g.,~\cite{KOS:08, DHK+:10, DRST14, HKM14, Kane11, Kane14b, Kane14}.
The algorithmic technique of~\cite{KKMS:08} 
achieves the information-theoretically optimal noise tolerance,
but it leads to algorithms with runtime $n^{\poly(1/\eps)}$ -- where $n$ is the dimension and 
$\eps$ the error rate -- even for a single LTF under the Gaussian distribution.
A sequence of subsequent works~\cite{KLS:09jmlr, Daniely15, ABL17}
focused on designing $\poly(n, 1/\eps)$ time learning algorithms that can tolerate constant noise rate 
in the malicious model (and the adversarial label noise model)
with respect to well-behaved continuous distributions.
The culmination of this line of work~\cite{ABL17} was that the concept class
of {\em origin-centered} LTFs is efficiently learnable 
in the malicious model -- with error guarantee of $O(\eps)$ for noise rate $\eps$ -- 
under {\em isotropic} (i.e., zero-mean, identity covariance) log-concave distributions. 


Perhaps surprisingly, the concept class of {\em origin-centered} LTFs is the only family of Boolean functions
for which polynomial-time algorithms are known in the malicious model. 
This motivates the following broad question that was posed as an open problem in previous work
~\cite{KLS:09jmlr, ABL17}:
\begin{quote}
{\em Are there computationally efficient learning algorithms in the malicious noise model -- 
with dimension-independent error guarantees -- for more general classes of Boolean functions?}
\end{quote}


\paragraph{Overview of Our Results.}
In this paper, we study this question with a focus on more general geometric concept  classes, 
namely low-degree PTFs and intersections of a constant number of LTFs. 
We provide new algorithmic and analytic techniques that yield 
the {\em first} polynomial-time PAC learning algorithms for these concept classes in the {\em nasty noise model}
(hence, in the malicious model as well)
with {\em dimension-independent error guarantees}. 

Specifically, we give a robust learning algorithm for low-degree PTFs in the nasty model 
that succeeds under a number of well-behaved distributions -- including the Gaussian distribution 
and, more generally, any log-concave distribution with (approximately) known low-degree moments. 
Prior to our work, no non-trivial efficient learning algorithm was known (even) for degree-$2$
PTFs in the (weaker) malicious noise model. (As an implication of our techniques,
we also obtain the first efficient learning algorithm with dimension-independent error in the nasty model 
for LTFs under the uniform distribution on the hypercube.)

For LTFs under the Gaussian distribution, using additional ideas, we give 
a polynomial-time algorithm that achieves error $O(\eps)$, where $\eps$ is the noise rate, 
i.e., it matches the information-theoretically optimal error, up to a constant factor. 
This is the first malicious/nasty learning algorithm for the class of {\em arbitrary} LTFs that 
achieves error $O(\eps)$ in polynomial time.
Our result improves on prior work by Awasthi {\em et al.}~\cite{ABL17} in two respects: 
First, \cite{ABL17} achieved an $O(\eps)$ error bound for the special case of {\em origin-centered} LTFs,
and second their algorithm applies to the weaker malicious/agnostic models. 
On the other hand, the $O(\eps)$ bound of~\cite{ABL17} 
holds for the more general family of isotropic log-concave distributions.

Our third result is a polynomial-time learning algorithm with dimension-independent error guarantees
for intersections of (any constant number of) LTFs in the nasty model under the Gaussian distribution.
To the best of our knowledge, no efficient algorithm (with non-trivial error guarantees) 
was previously known even for intersections of $2$ LTFs in the (weaker) malicious noise model.

{At the core of our results is an efficient algorithm 
to approximate the {\em low-degree Chow parameters}
of any bounded function {\em in the presence of nasty noise}. 
Roughly speaking, the low-degree Chow parameters
of a function $f$ under a distribution $D$ are the ``correlations'' of $f$ (with respect to $D$) 
with all low-degree monomials (see Section~\ref{ssec:prelims} for the formal definition). 
Our algorithm succeeds for a range of reasonable distributions $D$ satisfying
mild concentration bounds and moment assumptions. 
At a high-level, our robust (low-degree Chow parameter estimation) algorithm 
employs an iterative spectral technique for outlier detection and removal, 
inspired by recent work in robust unsupervised learning~\cite{DKKLMS16}.
Our technique filters out corrupted points relying on the concentration 
of carefully chosen {\em low-degree polynomials}.

Our robust learning algorithms for PTFs and intersections of LTFs use our
Chow-parameters estimation algorithm as a basic subroutine. That is, for both concept classes,
our algorithms proceed in two steps: (1) We start by approximating the ``low-degree'' Chow parameters 
of our function, and (2) We use our approximate Chow parameters from Step (1) to find a proper hypothesis
that is close to the target concept. 

The algorithm for Step (2) differs for PTFs
and intersections of LTFs. For degree-$d$ PTFs, we use the fact that approximations 
to the degree-$d$ Chow parameters information-theoretically approximately determines our function.
Given this fact, we leverage known algorithmic techniques~\cite{TTV:09, DeDFS14} that allow us to 
efficiently find an accurate proper hypothesis with approximately these Chow parameters.
For intersections of $k$ LTFs, we rely on approximations to the degree-$2$ Chow parameters. 
In this case, these parameters allow
us to reduce our $n$-dimensional learning problem to a $(k+1)$-dimensional problem that we can efficiently 
solve by a simple net-based method. The correctness of this scheme crucially relies on a novel 
structural result about intersections of LTFs under the Gaussian distribution that may be of
broader interest.}


\subsection{Noise Models} \label{ssec:noise-models} 
We start by recalling the standard (noiseless) PAC learning model~\cite{val84}:
Let $\mathcal{C}$ be a class of Boolean-valued functions over $\R^n$.
We assume that there exists a fixed distribution $D$ over $\R^n$ 
and an unknown target concept $f \in \mathcal{C}$. The learning algorithm
is given a set $S = \{ (x^{(i)}, y_i)\}$ of $m$ labeled examples $(x^{(i)}, y_i)$, $1 \leq i \leq m$, 
where for each $i$ we have that $x^{(i)} \sim D$, $y_i = f(x^{(i)})$, and the $x^{(i)}$'s are independent.
The goal of the algorithm is to output a hypothesis $h$ such that with high probability the error
$\Pr_{x \sim D} [h(x) \neq f(x)]$ is small.

In this work, we consider the problem of learning geometric
concepts in the {\em nasty (noise)} model~\cite{BEK:02}, which we now describe.
As in the PAC model, a distribution $D$ over $\R^n$ is selected and a target
concept $f \in \mathcal{C}$ is chosen. The critical difference with the noiseless PAC model 
lies in how the labeled examples given to the learning algorithm are constructed.
In the nasty model, the examples that the algorithm gets are generated 
by a powerful adversary that works according to the following steps: 
First, the adversary chooses $m$ examples independently according to the distribution $D$. 
Then the adversary, upon seeing the specific $m$ examples that were chosen 
(and using his knowledge of the target function, the distribution $D$, and the learning algorithm), 
is allowed to remove a fraction of the examples
and replace these examples by the same number of arbitrary examples of its choice.
The points not chosen by the adversary remain unchanged and are labeled by their correct
labels according to $f$. The modified labeled sample $S'$ of size $m$ is then 
given as input to the learning algorithm. The only restriction 
is that \new{the adversary is allowed to modify at most an $\eps$-fraction of examples.} 
As in the PAC model, the goal of the algorithm is to output a hypothesis $h$ such that with high probability the error
$\Pr_{x \sim D} [h(x) \neq f(x)]$ is as small as possible. The information-theoretically optimal error achievable
in the nasty model is \new{well-known} to be $\Theta(\eps)$~\cite{BEK:02}. 

We will be interested in designing efficient learning algorithms in the nasty model, 
i.e., algorithms with sample complexity and running time $\poly(n, 1/\eps)$
that achieve error $f(\eps)$, where $f: \R_+ \to \R_+$ is a function such that $\lim_{x \to 0+} f(x)=0$. 
In other words, we want the error guarantee of our learning algorithm to be {\em independent of the dimension}.
The ``golden standard'' in this setting is to achieve $f(\eps) = O(\eps)$, i.e., to match the information-theoretic
limit, up to a constant factor.

As is well-known, the nasty noise model generalizes both the malicious noise model~\cite{Valiant:85, KearnsLi:93}
and the agnostic (adversarial label) model~\cite{Haussler:92, KSS:94}. In the malicious model, each labeled
example is generated independently as follows: With probability $1-\eps$, a random pair $(x, y)$ is generated
where $x \sim D$ and $y = f(x)$; and with probability $\eps$ the adversary can output an arbitrary point 
$(x, y) \in \R^n \times \{-1, 1\}$. Each of the adversary's examples can depend on the state of the learning
algorithm and the previous draws of the adversary. Hence, in the malicious model, the adversary can 
add corrupted labeled samples but cannot remove good labeled examples. In the adversarial label noise model, 
the adversary can corrupt an $\eps$-fraction of the labels of $f$ under $D$, but cannot change the distribution 
$D$ of the unlabeled points.

\subsection{Previous Work} \label{ssec:previous-work}
We now summarize the prior work that is most relevant to the results of this paper.

As mentioned in the preceding discussion, the malicious noise model 
with respect to arbitrary distributions is known to be very challenging computationally. 
Even for the class of $n$-dimensional LTFs,  
the only known efficient algorithm~\cite{KearnsLi:93} achieves an error of $\Omega(\eps n)$, 
where $\eps$ is the noise rate. Improving on this bound has remained a challenge for a long time, 
and it was recently shown~\cite{Daniely16} that this holds for a reason: under plausible
complexity assumptions, no efficient algorithm can achieve error 
at most $1/2 - 1/n^c$, for some constant $c>0$, even if 
$\eps$ is an arbitrarily small constant.
 
Due to the computational difficulty of malicious learning under arbitrary distributions, 
research on this front has focused on well-behaved distributions. 
The prior results most relevant to this paper 
are the works of Klivans {\em et al.}~\cite{KLS:09jmlr} and Awasthi {\em et al.}~\cite{ABL17}.
Klivans {\em et al.}~\cite{KLS:09jmlr} studied the problem of learning origin-centered LTFs 
in the malicious and adversarial label noise models, when the distribution on the unlabeled samples
is uniform over the unit sphere or, more generally, an {\em isotropic} log-concave distribution.
\cite{KLS:09jmlr} gave the first polynomial-time algorithms for these problems with error guarantee
{\em poly-logarithmic} in the dimension. For the uniform distribution, their algorithm achieves 
error $O(\sqrt{\eps} \log (n/\eps))$, where $\eps$ is the malicious noise rate. 
Under isotropic log-concave distributions, they achieve error $O(\eps^{1/3} \log^2 (n/\eps))$. 
These bounds were subsequently significantly 
improved by~\cite{ABL17} who gave a $\poly(n, 1/\eps)$ time algorithm 
that learns to accuracy $O(\eps)$, 
where $\eps$ is the malicious/adversarial label noise rate, \new{for isotropic log-concave distributions.}
Origin-centered LTFs are the only concept class for which efficient malicious learning
algorithms were previously known.

{At the technical level, the algorithm of \cite{KLS:09jmlr} uses a simple outlier
removal method to approximate the degree-$1$ Chow parameters, 
and then finds an LTF with approximately these Chow parameters.
(That is, the high-level approach of our work 
for learning degree-$d$ PTFs is a broad generalization of the \cite{KLS:09jmlr} approach.)
It is worth noting that the outlier removal procedure of~\cite{KLS:09jmlr} is a weaker version of the 
filtering technique from~\cite{DKKLMS16}. On the other hand, the algorithm of~\cite{ABL17}
uses a soft outlier removal procedure together with localization. 
Instead of using degree-$1$ Chow parameters, \cite{ABL17} uses hinge-loss minimization, 
which can be solved via a convex program. 
}

The problem of learning intersections of $2$ LTFs under arbitrary distributions (without noise) 
is one of the most notorious open problems in computational learning theory: No efficient algorithm
is known despite decades of effort and it is a plausible conjecture that the problem may be intractable.
In a sequence of works, Vempala~\cite{Vempala10, Vempala10a} gave $\poly_k(n, 1/\eps)$ 
time algorithms to PAC learn intersections of $k$ LTFs under the Gaussian distribution on $\R^n$.
{As we will explain in Section~\ref{ssec:techniques}, 
our algorithm for learning intersections of LTFs in the nasty model 
has some similarities with~\cite{Vempala10a},
and can be roughly viewed as a robust version of this algorithm.}
It is also known~\cite{Baum:91, KlivansLT09} that intersections of $k=2$ origin-centered LTFs 
are efficiently PAC learnable under isotropic log-concave distributions. We note that
these algorithms work in the {\em noiseless} PAC learning model. In the agnostic model,
the $L_1$-regression algorithm~\cite{KKMS:08} can learn an intersection of $k = O(1)$ LTFs
under tame distributions, though its running time is $n^{\poly(1/\eps)}$, even for the case of $k=1$.
Prior to our work, we are not aware of any non-trivial algorithms for this concept class
in the adversarial label/malicious noise model that run in time $\poly(n, 1/\eps)$, even for $k=2$.

Finally, we remark that our work is related to a sequence of recent results 
on robust estimation in the unsupervised setting~\cite{DKKLMS16, DiakonikolasKKL16-icml, DKKLMS17}. 
Specifically, our general algorithm to approximate the low-degree Chow parameters with nasty noise
is inspired by the outlier removal technique of~\cite{DKKLMS16}. 
We  emphasize however that the setting considered here 
is vastly more general than that of~\cite{DKKLMS16}. 
As a result, a number of new conceptual and technical ideas are required, 
that we introduce in this paper.

\subsection{Preliminaries}  \label{ssec:prelims}

We record the basic notation and definitions used throughout the paper. 
For $n \in \Z_+$, we denote by $[n]$ the set $\{ 1, 2, \ldots, n\}$. 
For $d \in \Z_+$, a degree-$d$ polynomial threshold function (PTF) over 
$\R^n$ is a Boolean-valued function $f: \R^n \to \{ \pm 1\}$
of the form $f(x) = \sign(p(x))$, where $p : \R^n \to \R$ is a degree-$d$ polynomial with real
coefficients. For $d=1$, we obtain the concept class of Linear Threshold Functions (LTFs)
or halfspaces. An intersection of $k$ halfspaces is any function $f: \R^n \to \{ \pm 1\}$ 
such that \new{there exist $k$ LTFs $f_i$, $i \in [k]$, with $f(x) =1$ iff $f_i(x) =1$, for all $i$.} 
For a degree-$d$ polynomial $p : \R^n \to \R$, we denote by $\|p\|_2$
its $L_2$-norm, i.e., $\|p\|_2 = \E_{x \sim D}[p(x)^2]^{1/2}$, where the intended
distribution $D$ over $x \in \R^n$ will be clear from the context. We say that $p$
is {\em normalized} if $\|p\|_2 = 1$.

We now define the {\em degree-$d$ Chow parameters} of a function with respect
to a distribution $D$. To do so, we require some notation. 
Let $m(x)$ be the function that maps a vector $x \in \R^n$ to all the monomials of $x$ of degree at most $d$. 
Concretely,  let $\ba^1,\dots, \ba^\ell$ be an enumeration of all $\ba \in \N^n$ with $\|\ba\|_1 \leq d$. 
We set $m_i(x) \eqdef \prod_{j=1}^n x_j^{\ba^i_j}$, for all $1 \leq i \leq \ell$. 
Now $m$ is a function from $\R^n$ to $\R^{\ell}$ with $\ell \leq \new{(n+1)}^d$. 
Let $f: \R^n \to [-1, 1]$ be a bounded function over $\R^n$ and let $D$ be a distribution over $\R^n$.
The  {\em degree-$d$ Chow parameters} of $f$ with respect to $D$ are the $\ell$ numbers
$\E_{x \sim D}[f(x) m_i(x)]$, for $1 \leq i \leq \ell$.

\new{We will need an appropriate notion of approximation 
for the Chow parameters of a function $f:\R^n \to [-1, 1]$.
We say that a vector $v \in \R^{\ell}$ approximates the degree-$d$ Chow parameters of 
$f$ {\em within Chow distance $\delta$} if the following holds: 
For all normalized degree-$d$ polynomials $p: \R^n \to \R$
with $p(x) = \sum_{i=1}^{\ell} a_i m_i(x)$,
we have that $\left| L(p) - \E_{X \sim D}[p(X)f(X)] \right| \leq \delta$, 
where $L(p) = \sum_{i=1}^{\ell} a_i v_i $ 
is the corresponding linear combination of our approximations.
Let $p_1, \ldots, p_{\ell}$ be an orthonormal basis for the set of degree-$d$
polynomials under $D$.
We note that the previous definition is 
equivalent to the $\ell_2$-distance between the vectors
$\left( L(p_i) \right)_{i=1}^{\ell}$ and  $\left( \E_{X \sim D}[p_i(X)f(X)] \right)_{i=1}^{\ell}$
being at most $\delta$.} 

{We say that a set of labeled samples $S$ is {\em $\eps$-corrupted} if it is generated
in the nasty model at noise rate $\eps$, i.e., the adversary is allowed to corrupt
an $\eps$-fraction of samples.}

\subsection{Our Results} \label{ssec:results}

We start by stating our core efficient procedure
that approximates the low-degree Chow parameters
of any bounded function under tame distributions in the presence of nasty noise:

\begin{theorem} [Estimation of Low-Degree Chow Parameters with Nasty Noise] \label{thm:robust-chow-informal}
Let $f: \R^n \to [-1, 1]$. 
There is an algorithm which, given $d \in \Z_+$, $\eps > 0$, and a set $S$ of $\poly(n^d, 1/\eps)$
$\eps$-corrupted labeled samples from a distribution $D$ over $\R^n$, 
where $D$ is either (a) the standard Gaussian distribution $N(0, I)$ 
or the uniform distribution $U_n$ over $\{\pm 1\}^n$, 
or (b) any log-concave distribution over $\R^n$ with known moments of degree up to $2d$, 
runs in $\poly(n^d, 1/\eps)$ time and with high probability, 
outputs approximations of  $\E_{X \sim D}[f(X) m_i(X)]$ for all degree at most $d$ monomials $m_i(x)$, 
such that for any normalized degree-$d$ polynomial $p: \R^n \to \R$, 
the approximation of $\E_{X \sim D}[f(X) p(X)]$ given by the corresponding linear combination 
of these expectations has error at most $\eps \cdot O_d(\log(1/\eps))^{d/2}$ in case (a) and 
at most $\eps \cdot O_d(\log(1/\eps))^{d}$ in case (b). 
\end{theorem}

We note that Theorem~\ref{thm:robust-chow-informal} applies (with appropriate parameters)
to a wide range of distributions over $\R^n$, 
as its proof requires only mild tail bounds and moment assumptions. 
See Definition~\ref{def:reasonable} and Proposition~\ref{prop:generic-chow} 
for detailed statements. 

Our first PAC learning result is an efficient algorithm for low-degree PTFs in the 
nasty noise model:

\begin{theorem}[Learning Low-Degree PTFs with Nasty Noise] \label{thm:ptfs-informal}
There is a polynomial-time algorithm for learning degree-$d$ PTFs in the presence of nasty noise 
with respect to $N(0, I)$ or any log-concave distribution in $\R^n$ with 
known moments of degree at most $2d$. Specifically, if $\eps$ is the noise rate, the algorithm runs
in $\poly(n^d, 1/\eps)$ time and outputs a hypothesis degree-$d$ PTF $h(x)$ that with high probability satisfies
$\Pr_{X \sim D} [h(X) \neq f(X)] \leq \eps^{\Omega(1/d)}$, where $f$ is the unknown target PTF.
\end{theorem}

This is the first polynomial-time algorithm for learning degree-$d$ PTFs, for any $d>1$, 
in the malicious/nasty noise model with dimension-independent error guarantees. 
The algorithm of Theorem~\ref{thm:ptfs-informal} starts by approximating 
the degree-$d$ Chow parameters of our PTF $f$ using Theorem~\ref{thm:robust-chow-informal},
and then employs known techniques~\cite{TTV:09, DeDFS14} to 
find a PTF $h$ with approximately these Chow parameters. The fact that 
$\Pr_{X \sim D} [h(X) \neq f(X)]$ will be small follows from the simple fact
that, for the considered distributions, approximation in Chow distance implies approximation 
in $L_1$-distance. (This holds for distributions $D$ 
such that $p(D)$ has non-trivial concentration and anticoncentration 
properties for all degree-$d$ polynomials $p$.)

Note that the special case of Theorem~\ref{thm:ptfs-informal} for $d=1$ (LTFs)
is a generalization of~\cite{ABL17}, as our result applies to all LTFs (not necessarily origin-centered).
We only require knowledge of the first $2$ moments 
of the underlying log-concave distribution in this case, which 
is equivalent to assuming isotropic position
as is done in \cite{ABL17}.
For $d=1$ under isotropic log-concave distributions, 
the final accuracy of our algorithm will be $O(\sqrt{\eps})$, 
while \cite{ABL17} obtains an $O(\eps)$ error bound 
for {\em origin-centered} LTFs. 
Finally, we note that for $d=1$, Theorem~\ref{thm:ptfs-informal} 
also holds under {\em the uniform distribution on the hypercube},
with a quantitatively worse -- but still dimension-independent -- 
error of $2^{-\Omega(\sqrt[3]{\log(1/\eps)})}.$
This follows by using the structural result of~\cite{DeDFS14} relating 
closeness in Chow distance and $L_1$-distance in the Boolean domain.

We note that, for the case of LTFs under the Gaussian distribution, the algorithm of  
Theorem~\ref{thm:ptfs-informal} has final $L_1$-error of $O(\eps \sqrt{\log(1/\eps)})$ (see Corollary~\ref{weakLTFCor}).
For this setting, we can in fact obtain an efficient  algorithm with near-optimal error guarantee:

\begin{theorem} [Near-Optimally Learning LTFs with Nasty Noise] \label{thm:ltf-optimal-informal} 
There is a polynomial-time algorithm with near-optimal error tolerance 
for learning LTFs in the presence of nasty noise 
with respect to $N(0, I)$. Specifically, if $\eps$ is the noise rate, the algorithm runs
in $\poly(n, 1/\eps)$ time and outputs a hypothesis LTF $h(x)$ that with high probability satisfies
$\Pr_{X \sim D} [h(X) \neq f(X)] \leq O(\eps)$, where $f$ is the unknown target LTF.
\end{theorem}

Our algorithm for Theorem~\ref{thm:ltf-optimal-informal} starts from the 
$O(\eps \sqrt{\log(1/\eps)})$ approximate LTF of Theorem~\ref{thm:ptfs-informal} 
and uses a new twist of the localization technique of~\cite{ABL17} 
to reduce the error down to $O(\eps)$. We note that a number of new ideas
are required here to make this approach work for all LTFs, 
as opposed to only origin-centered ones, 
and to be able to handle nasty noise.

Our third algorithmic result gives the first efficient learning algorithm 
for intersections of LTFs in the malicious/nasty noise model:

\begin{theorem} [Learning Intersections of LTFs with Nasty Noise] \label{thm:intersections-informal}
There is a polynomial-time algorithm for learning intersections of any constant number of LTFs in the presence of nasty noise 
with respect to $N(0, I)$. Specifically, if $\eps$ is the noise rate, the algorithm runs
in $\poly_k(n, 1/\eps)$ time and outputs a hypothesis intersection of $k$ LTFs $h(x)$ that with high probability satisfies
$\Pr_{X \sim D} [h(X) \neq f(X)] \leq \poly(k) \cdot \poly(\eps)$, where $f$ is the unknown target concept.
\end{theorem}


For Theorem~\ref{thm:intersections-informal}, after approximating the degree-$2$
Chow parameters of $f$, we give a relatively simple method to reduce the problem down to $k$ dimensions.
The correctness of this dimension-reduction scheme makes essential use of the following new structural result, 
that we believe is of broader interest:

\begin{theorem}[Robust Inverse Independence for Intersections of LTFs] \label{thm:structural-ltfs-informal}
Let $f:\R^n \to \{0, 1\}$ be the indicator function of an intersection of $k$ LTFs. 
Suppose that there is some unit vector $v$ so that for  any degree at most $2$ polynomial $p$ 
with $\E[p(G)]=0$ and $\E[p^2(G)]=1$, $G \sim N(0, I)$, 
we have that $|\E[f(G)p(v \cdot G)]|<\delta$. 
Then, if $G$ and $G'$ are Gaussians that are correlated to be the same in the directions 
orthogonal to $v$ and independent in the $v$-direction, we have that 
$\E[|f(G)-f(G')|] \leq  \poly (\delta) \cdot \poly(k) \;.$
\end{theorem}

\subsection{Our Techniques}  \label{ssec:techniques}

In this section, we give a detailed outline of our techniques
in tandem with a comparison to previous work.

\paragraph{Robust Estimation of Low-Degree Chow Parameters.}
All our robust PAC learning results hinge on a new algorithm 
to approximate the degree-$d$ Chow parameters of any bounded function 
with respect to a sufficiently nice distribution $D$, 
even under noise in the nasty model (see Proposition~\ref{prop:generic-chow}).
Before we explain the ideas underlying this algorithm, 
we elaborate on the metric in which these approximations are guaranteed to be ``close''. 
To motivate our choice of metric, we will first discuss another 
interpretation of the degree-$d$ Chow parameters of a function $f$.
In particular, these parameters encode a linear functional mapping degree at most $d$ polynomials $p$ 
to the expectation $\E_{X\sim D}[p(X)f(X)]$. It is natural to put a norm on Chow parameters 
that is the dual of the $L_2$ norm on polynomials $p$ (with respect to $D$). 
In particular, when we say that we have approximated the degree-$d$ 
Chow parameters of $f$ to within error $\delta$, 
we will mean that we have found a linear functional $L$ mapping degree at most $d$ polynomials 
to real numbers so that for any normalized degree-$d$ polynomial $p$, 
we have that $|L(p)- \E_{X\sim D}[p(X)f(X)]| \leq \delta$.

We start by noting that if we had access to noiseless samples, 
the desired approximation would be easy to perform. 
In particular, we could take $L(p)$ to be the empirical expectation of $p(x)f(x)$, 
and then -- so long as $D$ satisfies even mild concentration bounds -- with sufficiently many samples 
it is straightforward to show that this will be a good approximation with high probability. 
It turns out that so long as we have reasonably good tail bounds for $p(D)$, 
this empirical approximation also works well even against noise in the adversarial label (agnostic) noise model. 
This holds essentially because changing the value of $f$ on a small number of samples 
can only have a large impact on the expectation of $p(x)f(x)$ if $p$ is especially large a decent fraction of the time.

The situation becomes substantially more challenging when the noise can adversarially corrupt the unlabeled
examples as well, and in particular in the nasty noise model. The essential problem here is that the error 
in the values allows an adversary to produce many sample values where $p(x)$ is unusually large for some particular $p$, 
and this will -- almost regardless of the labels of these points -- cause substantial errors in the empirical expectation of $p(x)f(x)$. 
In order to circumvent this obstacle, we will need a technique for detecting and removing these outliers, 
and for this we will make use of a ``filter'' technique 
inspired by recent work on robust distribution learning~\cite{DKKLMS16}. 

The basic idea here is that if the algorithm knew which polynomials $p$ 
the adversary was trying to corrupt, it could simply remove all of the sample points 
for which $p(x)$ was too large, thus removing these errors. 
Unfortunately, every sample $x$ will have $p(x)$ be abnormally large for some polynomials $p$, 
so our algorithm will need to find a way to identify particular polynomials 
for which our expectation may have been substantially corrupted. 
In order to achieve this, we note that since there must be many erroneous points for which $|p(x)|$ is large, 
this will cause the empirical expectation of $p^2(x)$ to be substantially larger than it should be. 
This anomaly can be detected (assuming that the algorithm knows good approximations to the true $2d^{th}$ moments of $D$) 
by a spectral technique, namely an eigenvalue computation. If such a $p$ is found then, assuming good tail bounds 
on the distribution of $p(D)$, the fact that we have many data points with much larger values of $p(x)$ than should be likely, 
will allow us to find a large set of samples most of which are corrupted. 
This step essentially produces a strictly cleaner version of our original corrupted sample set, 
and by iterating this algorithm we eventually reach a point where there are no longer any bad polynomials.
At this point, we can show that the empirical approximation of the Chow parameters will be accurate.

Although the basic intuition outlined above is well in line with recent works~\cite{DKKLMS16, DKKLMS17} 
making use of the filter technique, there are a few crucial technical differences in our setting. 
The first of these is that we are now working in a much more general context. 
Previous works tended to make very specific assumptions about the underlying distribution 
(e.g., Gaussian or balanced product distribution). Here, we are only making assumptions 
about tail bounds of {\em higher-degree polynomials}. 
Importantly, existing works typically only needed to ensure that the expectations 
of degree-$1$ and $2$ polynomials were correct, while 
in our setting we will inherently need to use filters dealing with polynomials of larger degrees. 
We also run into a new technical complication in the initial steps of the algorithm. 
In order to get the filter technique to work, we need to begin by throwing away all of the ``extreme'' outliers. 
This is required for somewhat technical reasons involving showing that a number of necessary concentration bounds hold. 
In previous works, the criteria for identifying these extreme outliers were generally fairly simple 
(e.g., throwing away a point being too far from the mean in some appropriate metric). 
However, in our case, we have less structure to deal with, 
and therefore need a somewhat more general criterion. 
In particular, we throw away outliers where $|p(x)|$ is too large for any normalized degree-$d$ polynomial $p$.

Our robust algorithm for low-degree Chow parameter estimation 
has immediate applications for robustly learning the Chow parameters over a distribution $D$, 
if $D$ is a Gaussian, Bernoulli, or log-concave distribution (where in the latter case, the algorithm 
must also know the low-degree moments of $D$). 
In the following paragraphs, we explain how to 
apply this algorithm as a core subroutine to robustly PAC learn geometric concept classes. 

\paragraph{Robust Learning for Low-Degree PTFs.}
One of the most natural geometric families of Boolean functions
to consider is that of PTFs. 
By classic structural results~\cite{Chow:61short, Bruck:90}, we know
that any degree-$d$ PTF is uniquely determined by its degree-$d$ Chow parameters. 
This suggests that if we can learn the degree-$d$ Chow parameters of a degree-$d$ PTF $f(x)$ to sufficient accuracy, 
then we may be able to use them to learn $f$ itself. 
In fact, by known algorithmic results~\cite{TTV:09, DeDFS14} we know that 
this is essentially the case -- but with one slight wrinkle. 
Since we will only have approximations to the degree-$d$ Chow parameters of $f$, 
we will  need a robust version of the~\cite{Chow:61short, Bruck:90} structural theorem. 
That is, we will need to know that if two degree-$d$ PTFs
have $L_1$-distance at least $\eps$, then they must have Chow distance at least $g(\eps)$ 
for some reasonably large error function $g$. 
While in the case of the uniform distribution over the hypercube
establishing such a result is still a challenging open question for $d>1$, 
it is not hard to prove good bounds when $D$ is a Gaussian, or more generally, a log-concave distribution. 
For these distributions, it is relatively easy to prove that an error function $g$ proportional to $\eps^{d+1}$ should suffice. 
This gives an algorithm for properly learning a PTF
over one of these distributions to error $\tilde O(\eps^{1/(d+1)})$, 
even with $\eps$ error in the nasty model.

We note that, for large constant $d$, one cannot expect to do substantially 
better than this bound using only an approximation of the degree-$d$ Chow parameters. 
This is because there are pairs of degree-$d$ PTFs 
for which this $\eps$ vs. $\eps^{1/d}$ type relation is nearly tight. 
This suggests some sort of ``integrality gap'' getting in the way:
No generic algorithm will be able to learn the low-degree 
Chow parameters of an $\eps$-noisy PTF  to error better than $\eps$, 
and no generic algorithm will be able to learn a degree-$d$ PTF 
to error better than $\eps^{1/d}$ from its degree-$d$ Chow parameters. 
However, this is not the case for the special case of linear threshold functions, 
where $L_1$-distance and Chow distance are indeed proportional.

\paragraph{Optimally Robust Learning of LTFs.}
For the case of LTFs, the relation between Chow distance and $L_1$-distance
allows for the possibility of a much better algorithm: 
that of learning LTFs to an optimal $O(\eps)$ error. 
In fact, we give such an algorithm over the Gaussian distribution.
We note that a naive application of the ideas of the previous paragraph 
is already sufficient to obtain an error of only $O(\eps\sqrt{\log(1/\eps)})$. 
Removing the final logarithmic term requires several new ideas. 
The overarching principle in our new algorithm 
is to use the localization technique of~\cite{ABL17}, 
though with slightly different technical backing.

Our algorithm will run an initial first pass to obtain an $O(\eps\sqrt{\log(1/\eps)})$ approximation to $f$. 
This step approximates $f$ by an LTF with separator given by some hyperplane $H$. 
We will then perform rejection sampling on our inputs in order to simulate samples from another Gaussian distribution 
centered around $H$. Learning $f$ with respect to this new input distribution will allow us to refine our original guess.

There are two major impacts of our restriction procedure. The first is that if most of the erroneous samples are near $H$, 
they might survive the rejection sampling process with higher probability than other points. 
This means that the fraction of errors in our simulated sample set may be much larger. 
To compensate for this though, this restriction will amplify the effect of small errors in $f$, 
as moving away from $H$ now much more quickly moves one away from the center of the distribution. 
This means that learning even rough information about the restriction of $f$ will give us useful information about the original problem. 
These two effects, as it turns out nearly cancel each other out, 
with the exception that the $\sqrt{\log(1/\eps)}$ term in the error becomes 
a  $\sqrt{\log(1/\delta)}$, where $\delta$ is the (now much larger) error rate for the restricted distribution. 
By iterating this technique with thinner and thinner restrictions, 
we can eventually converge on $f$ to an error of only $O(\eps)$.

\paragraph{Robust Learning of Intersections of LTFs.}
As a final application, we give a robust algorithm for learning intersections of LTFs
with respect to the Gaussian distribution. 
This algorithm is very different than the one for PTFs, 
as \new{it is not possible} to recover such a function from its low-degree Chow parameters directly. 
For this problem, we will need to make use of a somewhat different idea. 

The key insight is that if $f$ is the indicator function of an intersection of $k$ halfspaces, 
then $f$ only depends on $k$ linear functions of the input. If we could identify these directions, 
we could project our inputs down to a $k$-dimensional subspace 
and proceed by applying even relatively inefficient algorithms to learn a function on this low-dimensional space. 
In order to learn this subspace, we note that for any $v$ perpendicular to all directions of interest, 
$f(G)$ is uncorrelated with $p(v\cdot G)$ for any function (and in particular polynomial function) $p$. 
If we knew the degree-$2$ Chow parameters of $f$, this would imply that $v$ was a null-vector of the associated matrix. 
This would allow us to easily identify such vectors $v$.

In order to turn this into an algorithm, we will first need an inverse version of this theorem. 
Namely, that if for some vector $v$ that $f$ is uncorrelated with $p(v\cdot G)$ for all degree-$2$ polynomials $p$, 
we will need to know that $f$ is in fact independent of the $v$-direction. 
In fact, since we only know approximations to the Chow parameters, we will need a {\em robust} version of this statement. 
Namely if for all degree-$2$ polynomials $p$, we have that $f$ is {\em nearly} uncorrelated to $p(v\cdot G)$, 
that $f$ will be {\em nearly} constant in the $v$-direction. 
See Theorem~\ref{halfspaceStrcutureProp} for the technical statement of this result.

The aforementioned robust structural result allows a very natural algorithm: 
We start by learning approximations of the degree-$1$ and $2$ Chow parameters of $f$. 
We then let $V$ be the subspace spanned by the vector of degree-$1$ Chow parameters 
and the largest $k$ eigenvalues of the matrix corresponding to the degree-$2$ Chow parameters. 
It is not hard to see that $f$ is nearly uncorrelated to $p(v\cdot G)$ for any $v\perp V$. 
This along with the above structural result allows us to approximate $f(x)$ 
by a function that depends only on the projection $\pi_V(x)$, which as described above, can be learned by brute-force methods.

{We note that the algorithm of~\cite{Vempala10a} for finding the $k$-dimensional 
invariant subspace is similar to ours. Instead of considering 
the largest eigenvalues of the degree-$2$ Chow parameters, 
the algorithm of \cite{Vempala10a} relies on the smallest eigenvalues 
of the covariance of the positive samples, which is roughly equivalent.
The correctness of this algorithm uses the following lemma: 
in the $k$-dimensional subspace in which the intersection is non-trivial, 
the variance of the positive samples is less than one, which has some similarities
with our structural result.
The major difference is that our structural lemma is robust, 
and as a result our algorithm can tolerate nasty noise 
(using our approximations to the Chow parameters).
}


\subsection{Organization} \label{sec:structure}

The structure of this paper is as follows:
In Section~\ref{sec:robust-chow}, we give our algorithm to robustly estimate the low-degree Chow parameters of a bounded
function, thereby establishing Theorem~\ref{thm:robust-chow-informal}. In Section~\ref{sec:ptfs}, we describe the 
required machinery to prove Theorem~\ref{thm:ptfs-informal}. Section~\ref{sec:ltfs} proves our robust learning algorithm for LTFs
with near-optimal accuracy (Theorem~\ref{thm:ltf-optimal-informal}). Finally, in Section~\ref{sec:intersections}, we give our algorithm 
for robustly learning intersections of LTFs (Theorem~\ref{thm:intersections-informal}) 
and the associated structural result (Theorem~\ref{thm:structural-ltfs-informal}).

\section{Robust Estimation of Low-Degree Chow Parameters} \label{sec:robust-chow}

\subsection{Generic Algorithm} \label{ssec:generic}

In this section, we give our generic algorithm to robustly approximate
the degree-$d$ Chow parameters of any bounded function over $\R^n$.
Our algorithm succeeds for any distribution $D$ over $\R^n$ that satisfies
mild concentration and moment conditions. 
We will show that in order to approximate the Chow parameters of degree at most $d$,
it suffices to run a filter algorithm that attempts to make the moments of the distribution close to what they should be. 
To do this, it is enough to have approximations to the moments up to degree $2d$ 
and to know tail bounds for polynomials of degree at most $d$.

Specifically, we introduce the following definition:

\begin{definition}[Reasonable Distribution] \label{def:reasonable}
We say that a probability distribution $D$ over $\R^n$ is {\em reasonable} if it satisfies the following conditions:
\begin{itemize}
\item[(i)] ({\bf Concentration}) A tail bound for all degree at most $d$ polynomials: that is, 
a function $Q_d(T)$ such that for all polynomials $p(x)$ with $\|p\|_2 \leq 1$, 
$\Pr_{X \sim \D}[|p(X)| \geq T] \leq Q_d(T)$.

\item[(ii)] ({\bf Known Approximations of Low-Degree Moments})
A matrix $\Sigma$ such that $(1-\gamma) \E_{X \sim \D}[m(X) m(X)^T] \preceq \Sigma \preceq (1+\gamma) \E_{X \sim \D}[m(X) m(X)^T]$, 
for some relative error $\gamma > 0$ that is smaller than a sufficiently small constant.

\item[(iii)] A parameter $\delta>0$ that satisfies $\delta \geq \int_0^\infty T \min \{ \eps, Q_d(T)  \} dT$. 
Intuitively, the parameter $\delta$ is the maximum amount by which an $\eps$-probability mass can contribute to the $\E_{X \sim D}[p^2(X)]$.

\item[(iv)] A threshold $T_{\max}$ such that  $Q_d(T_{\max}/2\sqrt{\ell}) \leq \eps/(10\ell)$ 
and $T_{\max} \geq \sqrt{\ell}$. \new{We will be able to ignore points $x$ with $|p(x)| \geq T_{\max}$ .}
\end{itemize}
\end{definition}

We will see in the next section that many common distributions satisfy this definition (for appropriate parameters), 
including the Gaussian distribution, log-concave distributions, the uniform distribution over the hypercube, etc.

Now we can state the main proposition from which our main algorithmic applications 
will follow:

\begin{proposition} \label{prop:generic-chow}
Let $D$ be a reasonable distribution with known parameters $Q_d(T), \Sigma , \delta$, and $T_{\max}$.
There is an algorithm that, given $d \in \Z_+$, $\eps > 0$,  
and a set $S'$ of $\eps$-corrupted labelled samples from $D$ of size 
$|S'|=\Theta(n^d T_{\max}^4/\eps^2)$, runs in $\poly(|S'|)$ time, 
and with probability at least $9/10$ outputs approximations of 
$\E_{X \sim D}[f(X) m_i(X)]$ for all monomials $m_i(x)$ of degree at most $d$,
such that for any degree-$d$ polynomial $p(x)$ the approximation of $\E_{X \sim D}[f(X) p(X)]$ given by the corresponding linear combination 
of these expectations has error at most $O\left(\var_{X \sim D}[p(X)] \sqrt{\eps(\gamma+\delta+\eps)}\right)$. 
\end{proposition}

At a high-level, the algorithm works as follows: First, we pre-process our corrupted set of samples $S'$ using 
a basic pruning step. Specifically, we remove samples $x \in S'$ 
such that there is a polynomial of degree at most $d$ with $\|p\|_2=1$ and $|p(x)| \geq T_{\max}$. 
Our main algorithm is an iterative filtering procedure:
Using the largest eigenvalue and eigenvector of an appropriate matrix, 
we can detect whether  there is such a polynomial $p$ whose variance is bigger in $S'$ than $D$. 
If there is, we can use the tail bound $Q_d(T)$ to find a filter that throws out points where $|p(x)|$ is too large. 
If there is no such polynomial, then we show that the empirical Chow parameters suffice, so we output those. 
Formally,  the algorithm is the following:

\bigskip

\fbox{\parbox{6.1in}{
{\bf Algorithm} {\tt Robust-Chow-Parameters}\\

\begin{enumerate}
\item Remove all points $x$ from $S'$ that have $m(x)^T \Sigma^{-1} m(x) \geq T_{\max}^2/2$.
\item Repeat the following until no more points are removed from $S'$:
\begin{enumerate}
\item Compute the matrix $M = \Sigma^{-1/2}  \E_{X \in_u S'}[m(X) m(X)^T] \Sigma^{-1/2} $.
\item Approximate the largest eigenvalue  $\lambda^{\ast}$ of  $M  - I$ and the corresponding unit eigenvector $v^{\ast}$.
\item If $\lambda^{\ast} \leq  O(\gamma + \delta + \eps)$, break, i.e., \new{goto Step \ref{step:loop-exit}. }
\item Consider the polynomial $p^{\ast}(x) \eqdef (v^{\ast})^T \Sigma^{-1/2} m(x)$. 
\item Find $T >0$ such that 
$$\Pr_{X \in_u S'} \left[ |p^{\ast}(X)| \geq T \right] \geq 4Q_d(T) + 3\eps/T_{\max}^2 \; .$$
\item Remove from $S'$ all samples with $|p^{\ast}(x)| \geq T$.
\end{enumerate}
\item \label{step:loop-exit} Return $\E_{X \in_u S'}[f(X) m_i(X)]$, for all $1 \leq i \leq \ell$.
\end{enumerate}

}}

\medskip

\begin{remark}
{\em 
The algorithm as written assumes that $\Sigma$ is non-singular. 
If it is singular, we can find its null vectors. 
Each of these corresponds to a non-constant polynomial $p(x)$ with $\E_{X \sim D}[p(X)^2]=0$ 
and so with probability $1$, $p(x)=0$. If we pre-process by removing all points with $p(x) \neq 0$ for all such polynomials, 
then we can ignore these null-vectors. We can then replace all the inverses in the algorithm 
with Moore-Penrose pseudo-inverses and still get the same guarantee.}
\end{remark}

We will now require a definition of a good set, that is a set of points in $\R^n$ which 
satisfies a set of desired properties from a large enough set of random samples from $D$. 
There is a complication here because our assumptions on $\Sigma$ and $\delta$ 
only give bounds on moments of degree up to $2d$. To naively show that $\E_{X \in_u A}[p(X)^2]$, where $A$ is a set of
samples from $D$,  is close to $\E_{X \sim D}[p(X)^2]$, 
we would need bounds \new{on $\var_{X \in_u A}[p(X)^2]$. However, we assume nothing about} moments of degree $4d$. 
We can get round this by considering the properties of the set after we've thrown away outliers in our pruning step.

\begin{definition} \label{def:good-sample}
Let $f : \R^n \to [-1, 1]$.
We say that a set $S$ of points in $\R^n$ is $(\eps, f)$-good 
if for all polynomials $p(x)$ of degree at most $d$ with $\|p\|_2=1$ 
all the following conditions are satisfied:
\begin{itemize}
\item[(i)] For all $T>0$, we have that $\left|\Pr_{X \in_u S} \left[p(X) > T\right] - \Pr_{X \sim D}\left[p(X) > T\right]\right| \leq \eps/(10T_{\max}^2)$.
\item[(ii)] Let $S_{\mathrm{prune}}$ be the subset of points in $S$ that satisfy 
condition $\mathrm{prune}$, i.e., for all $x \in S_{\mathrm{prune}}$ it holds $m(x)^T \Sigma^{-1} m(x) \leq T_{\max}^2/2$.
Then, for all $T>0$, we have that 
$$\left|\Pr_{X \in_u S_{\mathrm{prune}}} [p(X) > T] - \Pr_{X \sim D|\mathrm{prune}}[p(X) > T]\right| \leq \eps/(10T_{\max}^2) \;.$$
\item[(iii)] It holds $\left| \E_{X \in_u S_{\mathrm{prune}}}[p(X)f(X)] - \E_{X \sim D|\mathrm{prune}}[p(X)f(X)] \right| \leq \eps$.
\end{itemize}
A set $S$ that satisfies conditions (i) and (ii) is called $\eps$-good.
\end{definition}

Before showing that a set of random samples is good, we need a couple of lemmas about our pruning process. 
Firstly, we show that our pruning indeed implies a bound on the value of the polynomials we consider:
\begin{lemma} \label{lem:bound-all-polys} 
Let $p(x)$ be a degree-$d$ polynomial with $\|p\|_2 \leq 1$ 
and $y \in \R^n$ be such that $m(y)^T \Sigma^{-1} m(y) \leq T_{\max}^2/(1+\gamma)$. 
Then, we have that $|p(y)| \leq T_{\max}$. 
In particular, this holds for all $y \in \R^n$ satisfying 
$m(y)^T \Sigma^{-1} m(y) \leq T_{\max}^2/2$.
\end{lemma}
\begin{proof}
We can write $p(x)= v^T \Sigma^{-1/2} m(x)$ for some $v \in \R^{\ell}$. 
Using the bounds on $\Sigma$, we can write
$$1 \geq \E_{X \sim D}[p(x)^2] = v^T \Sigma^{-1/2} \E_{X \sim D}\left[m(X) m(X)^T\right] \Sigma^{-1/2} v \geq v^T v / (1+\gamma) \;.$$
Now we have that 
$$ p(y) = v^T \Sigma^{-1/2} m(y) \leq 
\|v\|_2 \cdot  \|\Sigma^{-1/2} m(y)\|_2 \leq \sqrt{(1+\gamma)} \cdot T_{\max}/\sqrt{(1+\gamma)}  \leq T_{\max} \;, $$
as desired.
\end{proof}

We next need to show that the pruning step does not throw away too many points:
\begin{lemma} \label{lem:prune-reasonable}
We have that: $\Pr_{X \sim D} \left[m(X)^T \Sigma^{-1} m(X) \geq T_{\max}^2/2\right] \leq \eps/10$.
If $S$ is any set of points satisfying Condition (i) of Definition~\ref{def:good-sample}, then 
$\Pr_{X \in_u S}\left[m(X)^T \Sigma^{-1} m(X) \geq T_{\max}^2/2 \right] \leq \eps/5$. 
\end{lemma}
\begin{proof}
Let $p_1(x),\dots, p_\ell(x)$ be an orthonormal basis 
for the set of all polynomials of degree at most $d$ 
under the inner product $\E_{X \sim D}[p(X)q(X)]$ for polynomials $p(x), q(x)$.

From the definition of $T_{\max}$, we have $Q_d\left(T_{\max}/2\sqrt{\ell}\right) \leq \eps/(10 \ell)$.
Thus, the probability that for a given $i$, $1 \leq i \leq \ell$, 
it holds $|p_i(X)| \geq T_{\max}/(2\sqrt{\ell})$ for $X \sim D$ is at most 
$\eps/(10 \ell)$. By our assumption on $S$, the probability of the same event under $S$ is at most 
$\eps/(10 \ell) + \eps/(10 T_{\max}^2) \leq \eps/(5 \ell).$
By a union bound, the event that there exists $i$, $1\leq i \leq \ell$, 
such that $|p_i(X)| \geq T_{\max}/(2 \sqrt{\ell})$, 
for some $1 \leq i \leq \ell$, has probability at most $\eps/10$ under $ X \sim D$ and at most $\eps/5$ under $X \in_u S$.

Now fix an $x \in \R^n$ with $m(x)^T \Sigma^{-1} m(x) \geq T_{\max}^2/2$. 
Consider the polynomial $$p(y)=m(x)^T \Sigma^{-1} m(y)/ \|\Sigma^{-1/2} m(x) \|_2 \;.$$ 
Then we have that $p(x)= \|\Sigma^{-1/2} m(x) \|_2 \geq T_{\max}/\sqrt{2}$, and that
$$\E_{X \sim D}[p(X)^2] = m(x)^T \Sigma^{-1}  \E_{X \sim D}\left[m(X) m(X)^T\right] \Sigma^{-1} m(x)/\|\Sigma^{-1/2} m(x) \|_2^2 \;.$$
Recall that, by our assumption on $\Sigma$, we have
$(1+\gamma)^{-1}\Sigma \preceq \E_{X \sim D}[m(X) m(X)^T] \preceq (1-\gamma)^{-1}\Sigma$, 
and thus we have
$$1/(1+\gamma)\leq \E_{X \sim D}[p(X)^2] \leq 1/(1-\gamma).$$
Thus, we can write $p(x)=\sum_{i=1}^{\ell} a_i p_i(x)$, 
where $\|a\|_2 \leq 1/\sqrt{(1-\gamma)}$, and therefore $\|a\|_1 \leq \sqrt{\ell}/\sqrt{(1-\gamma)}$. 
If all $p_i$'s have $|p_i(x)| \leq T_{\max}/(2\sqrt{\ell})$, 
then we would have 
$$|p(x)| \leq \sqrt{\ell}/\sqrt{(1-\gamma)} \cdot T_{\max}/(2\sqrt{\ell})  < T_{\max}/(2 \sqrt{1-\gamma}) < T_{\max}/\sqrt{2} \;.$$
Since $p(x) \geq T_{\max}/\sqrt{2}$, one of these conditions must fail. 
However, we argued that this event happens with appropriately bounded probabilities
under both $S$ and $D$. This completes the proof.
\end{proof}

Now we can show that a large enough set of samples drawn from $D$ is $(\eps, f)$-good with high probability.
\begin{lemma} \label{lem:good-sample}
With probability $9/10$, if $S$ is a set of $\Omega(n^d T_{\max}^4/\eps^2)$ samples from $D$, 
then $S$ is $(\eps, f)$-good. 
\end{lemma}
\begin{proof}
To establish condition (i), we note that
the VC-dimension of the set of degree-$d$ PTFs is $O(n^d)$.
So, by the VC-inequality~\cite{DL:01}, with probability $99/100$, 
we have that 
$$\left| \Pr_{X \in_u S} [p(X) > T] - \Pr_{X \sim D}[p(X) > T] \right| \leq \eps/(10T_{\max}^2)$$
for all degree at most $d$ polynomials $p(x)$, and all $T \in \R$.
We henceforth condition on this event.

We now proceed to establish that Condition (ii) is satisfied.
Lemma \ref{lem:prune-reasonable} gives that, for $X \sim D$, 
we have $m(X)^T \Sigma^{-1} m(X) \geq T_{\max}^2/2$ 
with probability at most $\eps/10$. 
By our conditioning, the set $S$ satisfies Condition (i). 
Thus, Lemma \ref{lem:prune-reasonable} also implies that for $X \in_u S$,
we have that $m(X)^T \Sigma^{-1} m(X) \geq T_{\max}^2/2$
with probability at most $\eps/5$. 
Thus, $S_{\mathrm{prune}}$ contains $\Omega( n^d T_{\max}^4/\eps^2)$ samples 
that can be considered as being drawn from $D|\mathrm{prune}$. 
Condition (ii) now follows from the same argument as (i)
with probability at least $99/100$.

For Condition (iii), note that there exists a set of polynomials 
$p_1(x),\dots, p_\ell(x)$ that give an orthonormal basis 
for the set of all polynomials of degree at most $d$ under the inner product 
$\E_{X \sim D}[p(X)q(X)]$ for polynomials $p(x), q(x)$. 
Note that for any $i$, $1 \leq i \leq \ell$, we have 
$$\var_{X \sim D|\mathrm{prune}}[p_i(X) f(X)] \leq \var_{X \sim D}[p_i(X) f(X)]/(1-\eps/10) 
\leq (1+\eps) \E_{X \sim D} [p_i(X)^2 f(X)^2] \leq 1 + \eps \;,$$
where we used the fact that the range of $f$ is $[-1, 1]$.
Fix $i$, $1 \leq i \leq \ell$.
By Bernstein's inequality, since we have 
$|S_{\mathrm{prune}}| = \Omega(T_{\max}^4/\eps^2)$, we get that 
\begin{align*}
&\Pr\left[\left| \E_{X \in_u S_{\mathrm{prune}}}[p_i(X) f(X)] -  \E_{X \sim D|\mathrm{prune}}[p_i(X)f(X)] \right|\geq (\eps/T_{\max}) \right] \\
&= \Pr\left[ \left| S_{\mathrm{prune}} \right| \left| \E_{X \in_u S_{\mathrm{prune}}}[p_i(X) f(X)] 
- \E_{X \sim D|\mathrm{prune}}[p_i(X)f(X)] \right|\geq  \left|S_{\mathrm{prune}}\right| (\eps/T_{\max}) \right] \\
&\leq \exp\left(- \frac{(1/2) (|S_{\mathrm{prune}}| \eps/T_{\max})^2}{|S_{\mathrm{prune}}| (1+\eps) + (1/3) T_{\max} \cdot |S_{\mathrm{prune}}| (\eps/T_{\max})}\right) \\
& = \exp\left(-\frac{(1/2)|S_{\mathrm{prune}}| \eps^2}{T_{\max}^2 (1+4\eps/3)} \right) \\
& = \exp\left(-\Omega\left(T_{\max}^2 \right)\right) \\ 
&\leq 1/(100T_{\max}^2) \\ 
& \leq 1/(100\ell) \;.
 \end{align*}
By a union bound, we get that for all $i$, $1 \leq i \leq \ell$, it holds  
$$\left| \E_{X \in_u S_{\mathrm{prune}}}[p_i(X) f(X)] - \E_{X \sim D|\mathrm{prune}}[p_i(X)f(X)] \right| \leq \eps/T_{\max} \leq \eps/\sqrt{\ell} \;,$$ 
with probability at least $99/100$.
We condition on this event.
Note that any polynomial $p(x)$ with $\|p\|_2=1$ can be written as 
$p(x)=\sum_{i=1}^{\ell} a_i p_i(x)$, where $\|a\|_2=1$, 
and so $\|a\|_1 \leq \sqrt{\ell}$. 
Thus, any such $p(x)$ has 
$$\left| \E_{X \in_u S_{\mathrm{prune}}}[p_i(X) f(X)] - \E_{X \sim D|\mathrm{prune}}[p_i(X)f(X)] \right| 
\leq \sum_{i=1}^{\ell} |a_i| \eps/\sqrt{\ell} \leq \eps \;.$$
By a union bound, all the above $99/100$-probability events 
hold with probability at least $9/10$. This completes the proof.
\end{proof}

Now we can analyze the main loop of the algorithm. 
We either have that the empirical distribution has moments that well approximate 
those of $D$ or else the algorithm produces a filter that improves $S'$. 
Let $\Delta(G,S')$ be the size of the symmetric difference between $G$ and $S'$.
Then, it suffices to show that a single iteration satisfies the following:

\begin{proposition} \label{prop:loop}
If we run the main loop of the algorithm above on a set $S'$ 
of samples such that $\Delta(G, S') \leq 3\eps$ for some $\eps$-good set $G$, 
then either (a) we have that $\E_{X \in_u S'}[p(X)^2] \leq 1 + O(\gamma+\delta+\eps)$, 
for all polynomials $p(x)$ with degree at most $d$ that have $\|p\|_2=1$,
or else (b) the loop gives a set $S'' \subset S'$ with $\Delta(G,S'') \leq \Delta(G,S') - \eps/(10T_{\max}^2)$. 
\end{proposition}
\begin{proof}
The case when we exit the loop is simple. 
For every polynomial $p(x)$ with degree at most $d$ that has $\|p\|_2=1$, 
there is a  vector $v$ such that $p(x) =v^T \Sigma^{-1/2} m(x)$. 
Thus, we have  
$$1= \E_{X \sim D}[p(X)^2] = v^T \Sigma^{-1/2}  \E_{X \sim D}\left[m(X) m(X)^T\right] \Sigma^{-1/2} v \;.$$
Recalling that 
$$\Sigma^{-1/2}  \E_{X \sim D }[m(X) m(X)^T] \Sigma^{-1/2} \geq (1+\gamma)^{-1} I \;,$$ 
we deduce that $\|v\|_2^2 \leq 1 + \gamma$.

For any polynomial $p(x) = v^T \Sigma^{-1/2} m(x)$, we can write:
\begin{align*}
\E_{X \in_u S'}[p(X)^2] -1 
&= v^T \Sigma^{-1/2}  \E_{X \in_u S'}\left[m(X) m(X)^T\right] \Sigma^{-1/2} v - 1 \\
& = v^T M v - 1 \\
&\leq (1+\gamma) (v^{\ast})^T M v^{\ast} -1 =  \\
&\leq (1+\gamma)(1+\lambda^{\ast}) - 1 \\
& =  O(\gamma + \lambda^{\ast})\;.
\end{align*}
So, when $\new{\lambda^{\ast}} \leq O(\gamma+\delta+\eps)$, 
we have $\E_{X \in_u S'}[p(X)^2] \leq 1+ O(\gamma+ \delta+\eps)$ for all such $p(x)$.

\medskip

\noindent It remains to show that the algorithm produces a filter with the desired properties 
when $\new{\lambda^{\ast}} \geq \Omega(\gamma+\delta+\eps).$
Note that 
$$\|p^{\ast}\|_2^2 = (v^{\ast})^T \Sigma^{-1/2}  \E_{X \sim D} \left[ m(X) m(X)^T \right] \Sigma^{-1/2} v^{\ast} \;,$$ 
and so $(1+\gamma)^{-1} \leq \|p^{\ast}\|_2^2  \leq (1-\gamma)^{-1}$. 
On the other hand, we have $\E_{X \sim_u S'}[p^{\ast}(x)^2]= 1 + O(\gamma+\lambda^{\ast})$. 
We show that this is only possible when $\E_{X}[p^{\ast}(X)^2]$ 
is bigger under $S'$ than under $D$,  
and that under these circumstances, we there exists a valid threshold for our filter. 

Let $S$ be the subset of $G$ that contains the points $x$ satisfying 
$m(x)^T \Sigma^{-1} m(x) \leq T_{\max}^2 /2$. 
Then, we write $S' = S \cup E \setminus L$ for disjoint $E$ and $L$.
Thus, we have
\begin{equation} \label{eq:an-eq}
|S'| \cdot \E_{X \in_u S'}\left[p^{\ast}(X)^2\right] = |S| \cdot \E_{X \in_u S}\left[p^{\ast}(X)^2\right]
+ |E| \cdot \E_{X \in_u E}\left[p^{\ast}(X)^2\right] - |L| \cdot \E_{X \in_u L}\left[p^{\ast}(X)^2\right] \;.
\end{equation}

We start with the following simple lemma:

\begin{lemma} \label{lem:S-bound}
For all polynomials $p(x)$ with degree at most $d$ and $\|p\|_2=1$, 
we have $|\E_{X \in_u S}[p(X)^2] -1| \leq O(\eps+ \delta)$. 
\end{lemma}
\begin{proof}
We first show that conditioning on the pruning step 
does not change $\E[p(X)^2]$ much. Note that, using Lemma~\ref{lem:prune-reasonable}, 
we have $\dtv(D, D||\mathrm{prune}) \leq \eps/10$. We can write:
\begin{align*}
\left| \E_{X \sim D} \left[p(X)^2\right] -\E_{X \sim D|\mathrm{prune}} \left[p(X)^2\right] \right| 
&= \left| \int_{T=0}^{\infty} T \left( \Pr_{X \sim D|\mathrm{prune}} \left[|p(X)| > T\right] - \Pr_{X \sim D} \left[|p(X)| > T\right] \right) dT \right| \\
&\leq \int_{T=0}^{\infty} T \min \left\{ \eps/10,  Q_d(T) /(1-\eps/10) \right\} dT \\ 
&\leq 2 \delta \;.
\end{align*}
On pruned samples $x$, we have that $|p(x)| \leq  T_{\max}$ by Lemma \ref{lem:bound-all-polys}, 
and therefore
\begin{align*}
\left| \E_{X \in_u S}[p(X)^2] -\E_{X \sim D|\mathrm{prune}}[p(X)^2] \right| 
&=  \left| \int_{T=0}^{\infty} T \left( \Pr_{X \in_u S} [|p(X)| > T] - \Pr_{X \sim D|\mathrm{prune}}[|p(X)| > T] \right) dT \right| \\
& = \left| \int_{T=0}^{T_{\max}} T \left( \Pr_{X \in_u S}[|p(X)| > T] - \Pr_{X \sim D|\mathrm{prune}}[|p(X)| > T] \right) dT \right| \\
& \leq  \left| \int_{T=0}^{T_{\max}} T (\eps/T_{\max}^2) dT \right|  \\
&= \eps/2 \;,
\end{align*}
where we used that the set $S$ is the pruned set satisfying Condition (ii) of Definition~\ref{def:good-sample}.
The triangle inequality now gives that 
$$\left| \E_{X \in_u S}[p(X)^2] -1 \right|  = \left| \E_{X \in_u S}[p(X)^2] - \E_{X \sim D}[p(X)^2] \right|  \leq 2\delta + \eps/2 \;.$$
This completes the proof.
\end{proof}

\noindent 
We now show that the contribution of the set $L$ to the expectation of $p^2$ is small:

\begin{lemma} \label{lem:L-bound}
For all polynomials $p$ of degree at most $d$ with $\|p\|_2=1$, 
we have $|L| \cdot \E_{X \in_u L}[p(X)^2] \leq O(\delta + \eps) \cdot |S|$. 
\end{lemma}
\begin{proof}
Since $L \subset S$, for any event $A$, we have that 
$|L| \cdot \Pr_L[A] \leq |S| \cdot \Pr_S[A]$, and therefore 
$$\Pr_L[A] \leq \min \left\{1, (|S| / |L|) \cdot \Pr_S[A] \right\} \;.$$ 
Thus, we have the following sequence of inequalities:
\begin{align*}
(|L|/|S|) \cdot \E_{X \in_u L}[p(X)^2] 
&= \int_{T=0}^{T_{\max}} T (|L|/|S|) \Pr_{X \in_u L} \left[|p(X)| > T\right] dT \\
&\leq   \int_{T=0}^{T_{\max}} T \cdot \min \left\{ |L|/|S|, \Pr_{X \in_u S}\left[|p(X)| > T\right] \right\} dT \\
& \leq \int_{T=0}^{T_{\max}} T  \cdot \min \left\{ 3\eps, \Pr_{X \sim  D|\mathrm{prune}}\left[|p(X)| > T\right] + \eps/T_{\max}^2 \right\} dT \\
& \leq \int_{T=0}^{T_{\max}} T \cdot \min \left\{ 3\eps, (1+\eps) \cdot \Pr_{X \sim  D}\left[|p(X)| > T\right] + \eps/T_{\max}^2 \right\} dT \\
& \leq \int_{T=0}^{T_{\max}} T \cdot \min \left\{ 3\eps, (1+\eps) Q_d(T) \right\} dT  
+  \int_{T=0}^{T_{\max}} T (\eps/T_{\max}^2) dT \\
& \leq 3 \delta + \eps/2 \;.
\end{align*}
This completes the proof.
\end{proof}

As an immediate corollary, we obtain:
\begin{corollary} 
For all polynomials $p$ of degree at most $d$ with $\|p\|_2=1$, we have that 
$\E_{X \in_u S'}[p(X)^2] \geq 1 - O(\eps + \delta)$. 
\end{corollary}
\begin{proof} 
This follows from the equation for $\E_{X \in_u S'}[p(X)^2]$ 
similar to (\ref{eq:an-eq}), using Lemmas \ref{lem:L-bound} and \ref{lem:S-bound}, 
and the fact that $|E| \cdot \E_{X \in_u E}[p(X)^2] > 0$. 
\end{proof}

Our goal is to show that our algorithm will indeed find a filter in this case, i.e, 
there exists $T>0$ such that $\Pr_{X \in_u S'} \left[ |p^{\ast}(X)| \geq T \right] \geq 4Q_d(T) + 3\eps/T_{\max}^2$.
We will show this by contradiction using the following intermediate lemma:
\begin{lemma} \label{lem:if-no-filter}
If for all $T>0$, we have that 
$\Pr_{X \in_u S'}\left[ |p^{\ast}(X)| \geq T \right] \leq 4Q_d(T)+3\eps/T_{\max}^2$, 
then we have $|E| \cdot \E_{X \in_u E}[p^{\ast}(X)^2] \leq O(\gamma+\delta+\eps) \cdot |S'|$. 
\end{lemma}
\begin{proof}
Since $E \subset S'$, it follows that 
\begin{align*}
(|E|/|S'|) \cdot \Pr_{X \in_u E} [|p^{\ast}(X)| > T] 
&\leq \min \{ |E|/|S'|, \Pr_{X \in_u S'}\left[|p^{\ast}(X)| > T\right] \} \\
& \leq \min \left\{ 3\eps, 4Q_d(T)+3\eps/T_{\max}^2 \right\} \;.
\end{align*}
Since $\|p^{\ast}\|_2^2 \leq 1 + O(\gamma)$, 
by a similar proof to that in Lemma \ref{lem:L-bound} above, 
we have that
$$|E| \cdot \E_{X \in_u E}[p^{\ast}(X)^2] \leq O(\gamma+\delta+\eps) |S'| \;.$$
\end{proof}

Now we are ready to show that we do find a filter:

\begin{lemma} 
If $\lambda^{\ast} \geq \Omega(\gamma + \delta + \eps)$, 
then there exists a $T>0$ with  
$\Pr_{X \in_u S'}[|p^{\ast}(X)| \geq T] \geq 4Q_d(T) + 3\eps/T_{\max}^2$. 
\end{lemma}
\begin{proof}
We show the contrapositive. Suppose that there is no such $T$, 
then by Lemma \ref{lem:if-no-filter} we get that
$$|E| \cdot \E_{X \in_u E} [p^{\ast}(X)^2] \leq  O(\gamma+\delta+\eps) |S'| \;.$$
Now recall that $\|p^{\ast}\|_2^2 \leq 1 + O(\gamma)$. 
We can apply Lemma~\ref{lem:S-bound} to $p^{\ast}(x)/\|p^{\ast}\|_2$ 
to obtain
$$|S| \cdot \E_{X \in_u S}[p^{\ast}(X)^2] 
\leq (1+O(\gamma)) (1 + O(\delta+\eps)) |S| \;.$$
Using equation (\ref{eq:an-eq}) 
and the fact that $|L| \cdot \E_{X \in_u L}[p^{\ast}(X)^2] \geq 0$, 
we have
$$|S'| \cdot \E_{X \in_u S'}[p^{\ast}(X)^2] 
\leq (1+O(\gamma)) (1 + O(\delta+\eps))|S| + 
O(\gamma +\delta+\eps)|S'| 
\leq |S'| (1+O(\gamma + \delta + \eps)) \;.$$
However, this implies that 
$\lambda^{\ast} = \E_{X \in_u S'}[p^{\ast}(X)^2] -1 =  O(\gamma + \delta + \eps)$, 
yielding the desired contradiction.
\end{proof}

The algorithm thus finds a filter in this case. 
We next show that it rejects more points from $E$ than $S$, 
thus reducing $\Delta(S,S')$:

\begin{lemma} 
We have that $\Delta(S'',S) \leq \Delta(S',S) -\eps/(10T_{\max}^2)$. 
\end{lemma}
\begin{proof}
Using  the tail bound and the goodness of $S$, we obtain that
$$\Pr_{X \in_u S}\left[|p^{\ast}(x)| \geq T\right] \leq (1+\eps) Q_d(T) +  3\eps/T_{\max}^2 \;.$$
On the other hand, the filter rejects samples $x$ with $|p^{\ast}(x)| \geq T$ 
of which there are at least  $(4Q_d(T) +3\eps/T_{\max}^2)|S'|$ many in $S'$.
With appropriate choice of constant, we obtain that at least $2/3$ 
of the rejected samples are from $E$ and not $S'$. 
A similar analysis to Claim 8.12 of \cite{DKKLMS16} gives the lemma.
\end{proof}

Since neither $S''$ nor $S'$ contain any points $x$ with  
$m(x)^T \Sigma^{-1} m(x) \geq T_{\max}^2/2$, 
we also have $\Delta(S'',G) \leq \Delta(S',G) -\eps/(10T_{\max}^2)$. 
This completes the proof of Proposition \ref{prop:loop}.
\end{proof}

Now we analyze the case that we exit the loop.
Our aim is to show the following lemma:
\begin{lemma} \label{lem:chow-good} 
For any polynomial $p$ of degree at most $d$ with $\|p\|_2 \leq 1$, we have that
$$\left| \E_{X \in_u S'}[f(X)p(X)] - \E_{X \sim D}[f(X)p(X)] \right| 
\leq O(\sqrt{\eps(\gamma+\delta+\eps)}) \;.$$
\end{lemma}

Since the expectations the algorithm outputs are those over $S'$, 
Lemma~\ref{lem:chow-good} implies that the linear combinations 
that give an approximation to $\E_{X \sim D}[f(X)p(X)]$ 
have this error, and so the algorithm is correct.

To prove Lemma~\ref{lem:chow-good}, we will need to show a number of
intermediate statements. Firstly, we note that the pruning step does 
not affect this expectation under $D$ much:
\begin{lemma} \label{lem:pruning-doesnt-affect-chow}
For all polynomials $p$ of degree at most $d$ and $\|p\|_2=1$, we have:
$|\E_{X \sim D}[f(X)p(X)] - \E_{X \sim D|\mathrm{prune}}[f(X)p(X)]| \leq O(\eps + \sqrt{\eps\delta})$.
\end{lemma}
\begin{proof}
We have that 
$$\E_{X \sim D}[f(X)p(X)]= 
(1-\Pr_D[\mathrm{prune}]) \E_{X \sim D|\mathrm{prune}}[f(X)p(X)] 
+ \Pr_D[\mathrm{prune}] \E_{X \sim D|\neg \mathrm{prune}}[f(X)p(X)] \;.$$
Thus, we can write:
\begin{align*}
& \left| \E_{X \sim D}[f(X)p(X)] - \E_{X \sim D|\mathrm{prune}}[f(X)p(X)] \right| \\
&=  \left| \Pr_D[\mathrm{prune}] \E_{X \sim D|\mathrm{prune}}[f(X)p(X)] + 
\Pr_D[\mathrm{prune}] \E_{X \sim D|\neg \mathrm{prune}}[f(X)p(X)] \right| \\
&\leq \Pr_D[\mathrm{prune}] \sqrt{\E_{X \sim D|\mathrm{prune}}[p(X)^2]} 
+ \Pr_D[\mathrm{prune}] \sqrt{\E_{X \sim D|\neg \mathrm{prune}}[p(X)^2]} \\
& \leq \Pr_D[\mathrm{prune}] / (1-\Pr_D[\mathrm{prune}]) 
 + \Pr_D[\mathrm{prune}] \sqrt{\E_{X \sim D|\neg \mathrm{prune}}[p(X)^2]} \\
& \leq \eps + \Pr_D[\mathrm{prune}] \sqrt{\E_{X \sim D|\neg \mathrm{prune}}[p(X)^2]} \;.
\end{align*}
We need a bound on this last term, which we obtain as follows:
\begin{align*}
\Pr_D[\mathrm{prune}] \E_{X \sim D|\neg \mathrm{prune}}[p(X)^2] 
&= \E_{X \sim D}[p(X)^2] -\left(1-\Pr_D[\mathrm{prune}]\right) \E_{X \sim D|\mathrm{prune}}[p(X)^2] \\
&= \int_{T=0}^{\infty} T \cdot \left( \Pr_{X \sim D}\left[|p(X)| >T\right] 
- \left(1-\Pr_D[\mathrm{prune}]\right) \Pr_{X \sim D|\mathrm{prune}}\left[|p(X)| >T\right] \right) dT \\
&\leq  \int_{T=0}^{\infty} T \cdot \min \left\{ O(\Pr_D[\mathrm{prune}]), O(Q_d(T)) \right\} dT \\
& \leq O(\delta) \;.
\end{align*}
This gives that  $\Pr_D[\mathrm{prune}] \sqrt{\E_{X \sim D|\neg \mathrm{prune}}[p(X)^2]} = O(\sqrt{\eps\delta})$, 
as required.
\end{proof}

Finally, we can bound from above the contribution of the set $E$ to the expectation of $p^2$ when the algorithm terminates

\begin{lemma} \label{lem:E-bound} 
If $S'=S \cup E \setminus L$ is the final set of samples when the algorithm terminates, then
for all polynomials $p$ of degree at most $d$ and $\|p\|_2=1$, 
we have $|E| \cdot \E_{X \in_u E}[p(X)^2] \leq O(\gamma + \delta+\eps) \cdot |S'|$. 
\end{lemma}
\begin{proof}
Proposition~\ref{prop:loop} gives that 
$\left| \E_{X \in_u S'}[p(X)^2] -1 \right| \leq O(\gamma+\delta+\eps)$, Lemma \ref{lem:L-bound} gives that 
$|L| \cdot \E_{X \in_u L}[p(X)^2] \leq  O(\delta + \eps) \cdot |S|$, 
and Lemma \ref{lem:S-bound} gives $\E_{X \in_u S}[p(X)^2] \geq 1 - O(\eps+ \delta)$.
Thus, we have
\begin{align*}
|E| \cdot \E_{X \in_u E}[p(X)^2] 
& = |S'| \cdot \E_{X \in_u S'}[p(X)^2] + |L| \cdot \E_{X \in_u L}[p(X)^2] - |S| \cdot \E_{X \in_u S}[p(X)^2] \\
& \leq |S'| \cdot (1+ O(\gamma+\delta+\eps)) + |S| \cdot O(\delta+\eps) - |S| \cdot (1-O(\delta+\eps)) \\
& \leq  \left| |S'| - |S| \right|+ \left( |S| + |S'| \right)  \cdot O(\gamma + \delta+\eps) \\
&\leq O(\gamma + \delta+\eps) \cdot |S'| \;,
\end{align*}
recalling that $\Delta(S',S) \leq 2 \eps$.
\end{proof}

\noindent We are now ready to prove Lemma~\ref{lem:chow-good}.

\begin{proof}[Proof of Lemma \ref{lem:chow-good}:]
We have the following sequence of inequalities:
\begin{align*}
&|S'| \cdot  \left| \E_{X \in_u S}[p(X)f(X)] - \E_{X \in_u S'}[p(X)f(X)] \right| \\ 
&= \left| (|S'|-|S|) \cdot \E_{X \in_u S}[p(X)f(X)] + |L| \cdot \E_{X \in_u L}[p(X)f(X)] - |E| \cdot \E_{X \in_u E}[p(X)f(X)] \right| \\
& \leq \left| |S'|-|S| \right| \cdot \left| \E_{X \in_u S}[p(X)f(X)] \right| + |L| \cdot \left| \E_{X \in_u L}[p(X)f(X)] \right| + |E| \cdot \left| \E_{X \in_u E}[p(X)f(X)] \right| \\
& \leq \left| |S'|-|S| \right| \sqrt{\E_{X \in_u S}[p(X)^2]} + |L| \cdot \sqrt{\E_{X \in_u L}[p(X)^2]} + |E| \cdot \sqrt{\E_{X \in_u E}[p(X)^2]} \\
& \leq O(\eps|S'|) \cdot \sqrt{1+O(\delta+\eps)} + O(|S|\sqrt{\eps \cdot (\delta+\eps)}) + O(|S'|\sqrt{\eps \cdot (\gamma + \delta+\eps)}) \\
& \leq O(\sqrt{\eps(\gamma+\delta+\eps)})  \cdot |S'| \;,
\end{align*}
where the penultimate line uses Lemmas~\ref{lem:S-bound},~\ref{lem:L-bound}, and~\ref{lem:E-bound}.
\end{proof}

\subsection{Application of Generic Result to Tame Distributions} \label{ssec:concrete}

In this section, we show that a number of well-behaved distributions over $\R^n$
are reasonable, i.e., satisfy Definition~\ref{def:reasonable} with good parameters.
As a consequence, we obtain efficient robust estimators of the low-degree Chow parameters
for the corresponding distributions. In all cases, 
the robust estimators are obtained from Proposition~\ref{prop:generic-chow}
by plugging in the appropriate values of the parameters.

\paragraph{Standard Gaussian Distribution and Uniform Distribution over the HyperCube.}

For the standard $n$-dimensional Gaussian distribution $N(0, I)$ and the uniform
distribution $U_n$ over $\{\pm1\}^n$, we obtain the following corollary:

\begin{theorem} \label{thm:robust-chow-gaussian-uniform}
Let $f: \R^n \to [-1, 1]$ be a bounded function. 
There is an algorithm which, given $d \in \Z_+$, $\eps > 0$, 
and a set $S$ of $\eps$-corrupted labelled samples from either $D = N(0, I)$ or $D = U_n$
of size $\tilde{O}(n^3d/\eps^2)$, runs in $\poly(n^d, 1/\eps)$ time, 
and with probability at least $9/10$ outputs approximations of 
$\E[f(X) m_i(x)]$ for all monomials $m_i(x)$ of degree at most $d$ 
such that for any degree-$d$ normalized polynomial $p(x)$, 
the approximation of $\E_{X \sim D}\left[ f(X) p(X) \right]$ 
given by the corresponding linear combination of these expectations has error
at most $\sqrt{d} \eps \cdot O(d+\log(1/\eps))^{d/2}$. 
\end{theorem}

Theorem~\ref{thm:robust-chow-gaussian-uniform} follows immediately from Proposition~\ref{prop:generic-chow}
via the following standard concentration inequality:

\begin{fact} \label{thm:deg-d-chernoff}
Let $p: \R^n \to \R$ be a real degree-$d$ polynomial. 
Let $x$ be drawn from $D$,  where $D$ is either $N(0, I)$ or $U_n$. Then, for any
$T > e^d$, we have that $\Pr_{X \sim D} \left[ |p(X)| \geq T \|p\|_2 \right] \leq \exp(-\Omega(T^{2/d}))$.
\end{fact}

Finally, we note that similar bounds can be obtained for balanced product distributions over the hypercube, 
i.e., product distributions in which each coordinate is not too-biased towards $-1$ or $1$.

\paragraph{Log-concave Probability Distributions with Approximately Known Moments.}
A distribution over $\R^n$ with pdf $D$ is called log-concave if the function $\ln D$ is concave.
For log-concave distributions whose moments are approximately known, we obtain the following 
corollary of Proposition~\ref{prop:generic-chow}:

\begin{theorem} \label{thm:robust-chow-logconcave}
Let $f: \R^n \to  [-1, 1]$ be a bounded function.
There is an algorithm which, given $d \in \Z_+$, $\eps > 0$, and a set $S$ of $\eps$-corrupted 
labelled samples from a log-concave distribution $D$ over $\R^n$ 
with known moments of degree up to $2d$, 
of size $\tilde{O}(n^{3d}/\eps^2)$, runs in $\poly(n^d, 1/\eps)$ time, 
and with probability at least $9/10$ outputs approximations of 
$\E[f(X) m_i(x)]$ for all monomials $m_i(x)$ of degree at most $d$ 
such that for any normalized degree-$d$ polynomial $p(x)$, 
the approximation of $\E[f(X) p(x)]$ given by the corresponding linear combination 
of these expectations has error at most $\eps \cdot O(d+\log(1/\eps))^{d}$. 
\end{theorem}

Theorem~\ref{thm:robust-chow-logconcave} can be deduced from Proposition~\ref{prop:generic-chow}
via the following standard concentration inequality (see, e.g., Theorem 7 of \cite{CW:01}).

\begin{fact} \label{thm:deg-d-chernoff-log-concave}
Let $p:\R^n \to \R$ be a real degree-$d$ polynomial. 
Let $X$ be drawn from a log-concave distribution $D$ over $\R^n$. Then, for any
$T > e^d$, we have that $\Pr_{X \sim D}[ |p(X)| \geq T \|p\|_2 ]\leq \exp(-\Omega(T^{1/d}))$.
\end{fact}

\medskip

We now provide the details.
Note that $N(0, I)$ and $U_n$ satisfy Definition~\ref{def:reasonable} with 
$Q_d(T) = \exp(-\Omega(T^{2/d}))$ and $\gamma = 0$.
Indeed, for $D = N(0, I)$ we can calculate $\Sigma = \E_{X \sim D}[m(X)m^T(X)]$ exactly.

For $D = U_n$, observe that we only need to consider multilinear moments, in which case
it is easy to see that $\Sigma=I$: any monomial $m_i(x)$ takes values in $\{-1,1 \}$, 
so $\E_{X \sim D}[m_i(X)^2]=1$, and given two distinct monomials $m_i(x)$, $m_j(x)$, 
one of them contains a coordinate not appearing in the other.  
Thus, $\E_{X \sim D}[m_i(X)m_j(X)]=0$.
Because of the discrete setting, we can do better than the tail bound given by Fact~\ref{thm:deg-d-chernoff} for large thresholds. 
Since the samples are bounded, we can take $Q_d(T)=0$ for large enough $T$. 
In fact, we can see that $m(x)^T \Sigma^{-1} m(x)=\ell$ for all $x \in \{ \pm 1 \}^n$. 
This means that we can skip the pruning step entirely, 
and take $T_{\max}=\sqrt{\ell}$, 
since by Lemma \ref{lem:bound-all-polys} this is a bound for all polynomials we are interested in.

Similarly, log-concave distributions with known degree at most $2d$ moments 
satisfy Definition~\ref{def:reasonable} with 
$Q_d(T) = \exp(-\Omega(T^{1/d}))$ and $\gamma = 0$.

The following lemma completes the proof:


\begin{lemma} 
We have the following:
\begin{itemize}
\item[(i)] If $Q_d(T)= \exp(-\Omega(T^{2/d}))$, 
we can take $\delta= O(d(d+\ln(1/\eps)^d) \eps) $, $T_{\max}=O(nd\ln({n/\eps}))^{d/2}$.

\item[(ii)]  If instead $Q_d(T)= \exp(-\Omega(T^{1/d}))$, 
we can take $\delta= O((d+\ln(1/\eps)^{2d}) \eps)$, $T_{\max}=O(nd^2 \ln^2({n/\eps}))^{d/2}$.
\end{itemize}
\end{lemma}
\begin{proof}

We have $Q_d(T) = \eps/\sqrt{\ell}$, 
when $T= O(\ln(\sqrt{\ell/\eps}))^{d/2}$. 
So, for (i), we can take $T_{\max} = \sqrt{\ell} \cdot O(\ln(\sqrt{\ell/\eps}))^{d/2} = O(nd\ln(n/\eps))^{d/2}$, 
since $\ell \leq n^d$ for $n > 1$.
For (ii), we obtain $T_{\max} = \sqrt{\ell} \cdot O(\ln(\sqrt{\ell/\eps}))^{d} = O(nd^2\ln^2(n/\eps))^{d/2}$.

\new{Next, we obtain the bound on $\delta$ for (i).}
To get a bound on $\delta$, we will need the following technical claim:
\begin{claim}[see, e.g., Claim 7.18 from~\cite{DiakonikolasKS16c}] \label{clm:damn-integral}
For any $R > 0$, $d \in \Z_+, \eps > 0$,
and $\exp(-(a/R)^{2/d}) = \eps$, we have
$$\int_a^{\infty} \exp(-(T/R)^{2/d}) T dT \leq (d^2/2) \eps (d+\ln(1/\eps))^{d-1} \;.$$
\end{claim}

For some $R$, we have
\begin{align*}
\int_{T=0}^\infty T \min \{ \eps, Q_d(T)  \} dT = \delta & \geq \int_{T=0}^\infty T \min \{ \eps, \exp(-(T/R)^{2/d})  \} dT \\
& = \int_{T=0}^{R\ln(1/\eps)^{d/2}} \eps T dT + \int_{T=R\ln(1/\eps)^{d/2}}^{\infty} T   \exp(-(T/R)^{2/d}) dT \\
& \leq \eps R^2 \ln(1/\eps)^{d}/2 + (d^2/2) \eps (d+\ln(1/\eps))^{d-1} \\
& = O(d(d+\ln(1/\eps)^d) \eps) \;.
\end{align*}
Thus, we can take $\delta= O(d(d+\ln(1/\eps)^d) \eps)$.

The case when $Q_d(T)= \exp(-\Omega(T^{1/d}))$ is similar.
\end{proof}

\section{Robust Learning of Polynomial Threshold Functions under Tame Distributions} \label{sec:ptfs}
In this section, we show the following theorem, which is a detailed version of Theorem~\ref{thm:ptfs-informal}:

\begin{theorem}[Learning Low-Degree PTFs with Nasty Noise] \label{thm:ptfs-formal}
There is a polynomial-time algorithm for learning degree-$d$ PTFs in the presence of nasty noise 
with respect to $N(0, I)$ or any log-concave distribution in $\R^n$ with 
known moments of degree at most $2d$. Specifically, if $\eps$ is the noise rate, the algorithm uses 
a set of $\tilde{O}(n^{3d}/\eps^4)$ $\eps$-corrupted samples, runs
in $\poly(n^d, 1/\eps)$ time, and outputs a hypothesis degree-$d$ PTF $h(x)$ that with high probability satisfies
$\Pr_{x \sim D} [h(x) \neq f(x)] \leq \tilde{O}(d^2 \eps^{1/(d+1)})$, where $f$ is the unknown target PTF.
\end{theorem}





To prove our theorem, we need an efficient algorithm that starts with approximations to the low-degree Chow parameters
and computes approximations to the coefficients of the polynomial. 
This can be done by known techniques, as is implicit in prior work~\cite{TTV:09, DeDFS14} (see also~\cite{DDS12icalp}).

\new{
\begin{remark}
{\em For LTFs under the Gaussian distribution, there is a much simpler algorithm to post-process the approximate
Chow parameters obtained from Theorem~\ref{thm:robust-chow-gaussian-uniform} that gives a final 
$L_1$-error of $O(\eps \sqrt{\log(1/\eps)})$. See Corollary~\ref{weakLTFCor}.}
\end{remark}
}

To state the relevant result we will require to establish Theorem~\ref{thm:ptfs-formal}, 
we introduce some terminology:
The algorithm will use a variant of PTFs, which
we call \emph{polynomial bounded functions} (PBFs).  The projection
function $P_1: \R \to [-1,1]$ is defined by $P_1(t)=t$ for $|t| \leq 1$
and $P_1(t)=\sign(t)$ otherwise.  A PBF $g: \{-1,1\}^n \to [-1,1]$ is a
function $g(x)=P_1(q(x))$, where $q: \R^n \to \R$ is a real degree-$d$ polynomial.


\new{
Let $\mathcal{L}$ be a family of polynomials $\ell: \R^n \to \R$ that give a basis 
for all polynomials of degree at most $d$. Then we say that two sets of  Chow parameters 
$a_\ell, b_\ell$ for $\ell \in \mathcal{L}$ are $\eps$-close in $\ell_2$-distance 
if for all polynomials $p(x)$ of degree at most $d$ and $\|p\|_2 \leq 1$,  
for the $c_\ell$ that satisfy $p(x) = \sum_{\ell \in \mathcal{L}} c_\ell \ell(x)$, 
we have that $\left|\sum_{\ell \in \mathcal{L}} c_\ell (a_\ell-b_\ell)\right| \leq \eps$. }
We have the following statement, which is implicit in~\cite{TTV:09, DeDFS14}:

\new{
\begin{theorem} \label{thm:chow-to-weights}
Let $D$ be a distribution on $\R^n$, 
$f : \R^n \to [-1,1]$ be a bounded function, and $\mathcal{L}$ be as above. 
There is an algorithm with the following properties:  Suppose the algorithm is given as input a list 
$(a_{\ell})_{\ell \in \mathcal{L}}$ of real values and a parameter $\xi>0$ 
such that the Chow parameters $(a_{\ell})_{\ell \in \mathcal{L}}$ and 
$(\E_{X \sim D} [f(X) \ell(X)])_{\ell \in \mathcal{L}}$ are $\xi$-close in $\ell_2$.
The algorithm then outputs a function $h : \R^n \to [-1,1]$ with the following properties:

\begin{enumerate}
\item [(i)] The Chow parameters $(\E_{X \sim D} [h(X) \ell(X)])_{\ell \in \mathcal{L}}$ 
and  $(\E_{X \sim D} [f(X) \ell(X)])_{\ell \in \mathcal{L}}$ are $O(\xi)$-close in $\ell_2$.

\item [(ii)] $h(x)$ is of the form $h(x) = P_1(\sum_{\ell \in \mathcal{L}} w_{\ell} \ell(x))$.
\end{enumerate}
The algorithm runs for {$O(1/\xi^2)$} iterations, where in each iteration 
it estimates the Chow parameters $(\E_{X \sim D}[h'(X)\ell(X)])_{\ell \in \mathcal{L}}$ 
to be $O(\xi)$-close in $\ell_2$. Here, each $h'$ is a function of the form
$h'(x)=P_1( {\frac \xi 2} \cdot \sum_{\ell \in \mathcal{L}} v_{\ell} \ell(x))$, 
where the $v_\ell$'s are integers whose absolute values sum to $O(1/\xi^2)$.
\end{theorem}
}

In our application of the above theorem, the set $\mathcal{L}$ consists of the monomial functions $m_i(x)$, 
whereas the numbers $\alpha_0, \ldots, \alpha_{\ell}$ 
are the correlations of $f$ with $m_i$'s under the distribution $D$. 

We note that the above theorem is not explicitly stated in the above form
in previous work, but it follows easily from their proofs.

\new{To apply this algorithm, we need to show that we can get approximations 
to the Chow parameters not only of $f(x)$ but of each $h'(x)$ as well. 
We will use one of Theorems \ref{thm:robust-chow-gaussian-uniform} or \ref{thm:robust-chow-logconcave} 
to estimate the Chow parameters $\E_{X \sim D} [f(X) m_i(X)]$ to within error 
$\xi = \sqrt{d} \eps \cdot O(d+\log(1/\eps))^{d/2}$ or $\xi=\eps \cdot O(d+\log(1/\eps))^{d}$ respectively, 
using $N=\tilde{O}(n^{3d}/\eps^2)$ labeled $\eps$-corrupted samples from $D$. 
We can use the same algorithm to obtain estimates of $(\E_{X \sim D}[h'(X)\ell(X)])_{\ell \in \mathcal{L}}$ 
for the known PBFs $h'(x)$ used by the algorithm. In each iteration, we take $N$ 
$\eps$-corrupted samples from $D$ and label them according to the current hypothesis PBF $h'(x)$. 
Note that these labeled samples are still $\eps$-corrupted and so the requirements of 
Theorem \ref{thm:robust-chow-gaussian-uniform} or \ref{thm:robust-chow-logconcave} are satisfied. 
Thus, the approximations to the Chow parameters $(\E_{X \sim D}[h'(X)\ell(X)])_{\ell \in \mathcal{L}}$ 
the algorithm outputs are $\xi$-close in $\ell_2$-distance to the true parameters, as required. 
The algorithm uses $O(N/\xi)^2=O(N/\eps^2)= \tilde{O}(n^{3d}/\eps^4)$ 
$\eps$-corrupted samples in total.}

We have that $f(x)$ and $h(x)$ have Chow distance at most 
$\sqrt{d} \eps \cdot O(d+\log(1/\eps))^{d/2}$ or $\eps \cdot O(d+\log(1/\eps))^{d}$. 
We need to prove a bound on the $L_1$-distance.

For log-concave distributions, including the Gaussian, we will use:
\begin{lemma} \label{lem:Chow-to-Hamming} Let $D$ be a log-concave distribution. 
Let $f(x)$ be a degree-$d$ PTF and $g:\R^n \rightarrow [-1,1]$ be a bounded function
whose Chow parameters are $\eps$-close in $\ell_2$. 
Then, $\E_{X \sim D}[|f(X)-g(X)|]$ is at most $O(d \eps^{1/(d+1)})$. 
\end{lemma}
\begin{proof}
We can write $f(x)=\sgn(p(x))$, where $p(x)$ is a degree at most $d$ polynomial. Let $X$ be distributed as $D$. 
We have that $|\E[(f(X)-g(X))p(X)]| \leq \eps \|p\|_2$. 
Note that if $f(x)-g(x)\neq 0$ and $f(x) \neq 0$, 
since $|f(x)|=1$ and $|g(x)| \leq 1$, we get that $\sgn(f(x)-g(x))=\sgn(f(x))=\sgn(p(x))$. 
Thus, $(f(x)-g(x))p(x) = |f(x)-g(x)||p(x)| \geq 0$. 
Now we have that $\E[|f(X)-g(X)||p(X)|] \leq \eps \|p\|_2$. 
Let $\delta$ be $\E[|f(X)-g(X)|]$. Then, by Theorem 8 of \cite{CW:01}, we have
$$\Pr[|p(X)| \leq 3\eps \|p\|_2 /\delta] \leq O(d (3\eps/\delta)^{1/d}) \;.$$
Suppose for a contradiction, that this probability is smaller than $\delta/4$. 
Then, for any $t$, if $\Pr[|f(x)-g(x)| \geq t] > \delta/4$, 
then by a union bound with probability at least 
$\Pr[|f(x)-g(x)| \geq t]-\delta/4$, we have both 
$|f(x)-g(x)| \geq t$ and $|p(x)| > 3\eps \|p\|_2 /\delta$, 
and so $|f(x)-g(x)||p(x)| \geq 3\eps \|p\|_2 t /\delta$. 
In summary, for any $t > 0$, we have
$$\Pr[|f(x)-g(x)||p(x)| \geq 3\eps \|p\|_2 t /\delta] \geq \Pr[|f(x)-g(x)| \geq t] - \delta/4 \;.$$
Thus, we have
\begin{align*}
\delta & = \E[|f(x)-g(x)|] = \int_0^2 \Pr[|f(x)-g(x)| \geq t] dt \\
& \leq \int_0^2 \left( \delta/4 + \Pr[|f(x)-g(x)||p(x)| \geq 3\eps \|p\|_2 t /\delta] \right)  dt\\
& = \delta/2 + \int_0^2 \Pr[|f(x)-g(x)||p(x)| \geq 3\eps \|p\|_2 t /\delta]  dt \\
& = \delta /2 + \delta/(3 \eps \|p\|_2)  \int_0^{6\eps\|p\|_2/\delta} \Pr[|f(x)-g(x)||p(x)| \geq T] dT \\
& \leq \delta /2 + \delta/(3 \eps \|p\|_2)  \int_0^{\infty} \Pr[|f(x)-g(x)||p(x)| \geq T] dT \\
& \leq \delta /2 + \delta/(3 \eps \|p\|_2)  \cdot \eps \|p\|_2 \leq 5\delta/6 \;,
\end{align*}
which is our contradiction. Therefore, 
$$\delta/4 = O(d (3\eps/\delta)^{1/d}) \;.$$
Rearranging gives 
$3\eps/\delta = \Omega(\delta/d)^d$ and $\delta = O(d \eps^{1/(d+1)})$, 
which completes the proof.
\end{proof}

Now we note that if a PBF is close then so is the corresponding PTF.
\begin{lemma} \label{lem:PBF-to-PTF} 
If $f(x)$ is a degree-$d$ PTF, $g(x)=P_1(p(x))$, and 
$g'(x)=\sgn(p(x))$ for some function $p(x)$, 
then the $L_1$-distance between $f$ and $g'$, 
$\E[|f(X)-g'(X)|]/2$, is at most 
$\E[|f(X)-g(X)|]$. 
\end{lemma}
\begin{proof} 
Note that $g'(x)=\sgn(g(x))$. 
Thus, when $f(x) \neq g'(x)$, we have $\sgn(f(x)) \neq \sgn(g'(x))$. 
However, with probability $1$, $|f(x)|=1$, and then 
when $f(x) \neq g'(x)$, $|f(x)-g(x)| \geq 1$. 
Thus, we have $|f(X)-g'(X)| \leq 2 |f(x)-g(x)|$. 
Taking expectations under $D$ gives the lemma. 
\end{proof}

\begin{proof}[Proof of Theorem \ref{thm:ptfs-formal}]
As explained above, we use Theorems \ref{thm:robust-chow-gaussian-uniform} 
or \ref{thm:robust-chow-logconcave} to approximate the Chow parameters of $f(x)$ 
and as a subroutine in Theorem \ref{thm:chow-to-weights} to produce a 
PBF $h(x)=P_1(p(x))$ of degree at most $d$, 
which has Chow distance from $f(x)$ at most $\sqrt{d} \eps \cdot O(d+\log(1/\eps))^{d/2}$, if $D=N(0,I)$, 
or $\eps \cdot O(d+\log(1/\eps))^{d}$ otherwise.
By Lemma \ref{lem:Chow-to-Hamming}, we have that 
$\E_{X \sim D}[|f(X)-h(X)|] \leq O(d \eps^{1/(d+1)} \cdot O(d+\log(1/\eps)))$ in either case. 
By Lemma \ref{lem:PBF-to-PTF}, we also have that the $L_1$-distance 
$\E_{X \sim D}[|f(X)-h'(X)|]/2 \leq \tilde{O}(d^2 \eps^{1/(d+1)})$, 
where $h'(x)=\sgn(p(x))$ is a degree-$d$ PTF. 
We output this $h'(x)$.
\end{proof}

For the uniform distribution on $\{-1,1\}^n$, 
we obtain $L_1$-distance $2^{-\Omega(\sqrt[3]{\log(1/\eps)})}$ for the case $d=1$ 
by using a similar argument except using Theorem 7 of \cite{DeDFS14} 
in place of Lemma \ref{lem:Chow-to-Hamming}.

\section{Optimally Robust Learning of LTFs under the Gaussian Distribution} \label{sec:ltfs}

\newcommand{\erf}{\mathrm{erf}}
\newcommand{\polylog}{\mathrm{polylog}}

In this section, we prove the following theorem, a restatement of Theorem~\ref{thm:ltf-optimal-informal}:

\begin{theorem}[Near-Optimally Learning LTFs with Nasty Noise]  \label{LTFThrm} 
There is a $\poly(n/\eps)$ time algorithm that learns arbitrary LTFs under 
the standard Gaussian distribution on $\R^n$ to error $O(\eps)$ in the presence of nasty noise at rate $\eps$.
\end{theorem}

\new{
In the subsequent discussion, all probabilities and expectations are with respect to the standard 
$n$-dimensional Gaussian distribution, $N(0, I)$, unless otherwise specified. We use $G(x)$ to denote the pdf
of the standard one-dimensional Gaussian distribution.}

Before we get into the proof of Theorem \ref{LTFThrm}, 
we begin with some preliminary discussion. 
Note that any non-constant LTF over $\R^n$ can be expressed uniquely in the form
$$
f(x) = \sgn(v\cdot x + \theta) \;,
$$
for some unit vector $v$ and real number $\theta$. 
We call $v$ the \emph{defining vector} and we call $\theta$ the \emph{threshold}.

We first point out that the threshold of an LTF is easy to approximate from samples. 
In particular, we have that $\E[f]=\erf(\theta)$. Since the expectation of $f$ can be computed 
to $O(\eps)$ error even in the presence of noise in the nasty model, 
this allows us to compute an approximation $\theta_0$ to $\theta$ 
so that $\erf(\theta_0)-\erf(\theta)=O(\eps)$. 
We note that replacing $\theta$ by $\theta_0$ in the definition of our LTF 
introduces an error of only $O(\eps)$. Therefore, up to this additional $O(\eps)$ error, 
we may assume that the threshold of the function we are trying to learn is known to the algorithm. 
As it will simplify our analysis, we will therefore treat $\theta$ as if it were known to our algorithm.

Next, in order learn our threshold up to a given error, 
we will need to have a better idea of how much 
an error in our parameters contributes to an error in our function. We prove the following:
\begin{lemma}\label{LTFDistLem}
Given two LTFs $f(x)=\sgn(v\cdot x+\theta)$ and $g(x)=\sgn(w\cdot x+\theta)$ with the same threshold, 
we have that
$$
\|f-g\|_1 = O(\|v-w\|_2 G(\theta)) \;.
$$
\end{lemma}
\begin{proof}
We first note that it suffices to prove this result for small values of $\|v-w\|_2$, 
as we can take a path from $v$ to $w$ consisting of small steps, 
the sum of whose lengths is $O(\|v-w\|_2)$. 
Thus, we consider $w=\frac{1}{\sqrt{1+\gamma^2}}(v+\gamma u)$ for some $u\perp v$. 
We note that, up to $O(\gamma^2)$ error, 
we may replace $g(x)$ by $h(x)=\sgn((v+\gamma u)\cdot x + \theta)$. 
We note that $f$ and $h$ only differ when $|v\cdot x +\theta| \leq  \gamma u\cdot x$. 
As $v\cdot x$ and $u\cdot x$ are independent Gaussians, we have that the probability of this event equals
$$
P(\gamma) := \int_{-\infty}^\infty \int_{\theta-\gamma s}^{\theta+\gamma s} G(s) G(t) dt ds \;.
$$
Notice that the derivative of $P$ at $0$ is given by
$$
P'(0) = \int_{-\infty}^\infty 2 s G(s) G(\theta) ds = 2 G(\theta) \;.
$$
Therefore, $P(\gamma) = 2 G(\theta)\gamma + o(\gamma)$ as $\gamma\rightarrow 0$, 
and thus for sufficiently small $\gamma$, $\|f-g\|_1 = O(\gamma G(\theta))$. 
This completes our proof.
\end{proof}

As our main technique is to learn via the Chow parameters, 
we will want to know the relationship between our threshold function 
and its Chow parameters. In particular, we have:

\begin{lemma}\label{ChowLTFLem}
The degree-$1$ Chow parameters of the LTF 
$f(x)=\sgn(v\cdot x + \theta)$ with $\|v\|_2=1$ 
are $2G(\theta)v$.
\end{lemma}
\begin{proof}
It is clear that $\E[f(G)(w\cdot G)]=0$ for all $w\perp v$. 
Thus, we only need to evaluate $\E[f(G)(v\cdot G)]$. It is easy to see that this is
\begin{align*}
\int_{-\infty}^\infty \sgn(t+\theta)tG(t)dt & = \int_{-\theta}^\infty 2t G(t) dt = 2G(\theta) \;.
\end{align*}
This completes our proof.
\end{proof}

Combining this with Lemma \ref{LTFDistLem} and Theorem \ref{thm:robust-chow-gaussian-uniform}, 
we easily obtain the following pair of corollaries:
\begin{corollary}\label{chowToLTFCor}
There is an algorithm that given an $\eps$-approximation to the degree-$1$ Chow parameters 
of an LTF along with an $\eps$-approximation of its expectation, yields an $O(\eps)$-approximation of the function.
\end{corollary}
\begin{proof}
Let $u$ be our approximation of the degree-$1$ Chow parameters.
By Lemma \ref{ChowLTFLem}, if the true threshold is $\theta$, $u/\|u\|_2$ is within 
$(\eps /2) G(\theta)$ of the defining vector of the LTF. 
Therefore, by Lemma \ref{LTFDistLem}, we have that $f$ is within $O(\eps)$ of $\sgn(u/\|u\|_2\cdot x +\theta)$. 
By replacing $\theta$ by $\theta' = \erf^{-1}(m)$, where $m$ is our approximate expectation of $f$, 
we introduce another $O(\eps)$ error. This completes the proof.
\end{proof}

\begin{corollary}\label{weakLTFCor}
There exists an algorithm to learn a linear threshold function to error 
$O(\eps \sqrt{\log(1/\eps)})$, where $\eps$ is the noise rate in the nasty model, 
using $\poly(n/\eps)$ time and samples.
\end{corollary}
\begin{proof}
The algorithm is as follows:
\begin{enumerate}
\item Take $O(1/\eps^2)$ samples to obtain an $\eps$-approximation, $m$, of $\E[f]$.
\item Using Theorem \ref{thm:robust-chow-gaussian-uniform}, 
compute $u$, an $O(\eps\sqrt{\log(1/\eps)})$-approximation to the degree-$1$ Chow parameters of $f$.
\item Apply Corollary \ref{chowToLTFCor}.
\end{enumerate}
\end{proof}

Although this algorithm only learns to error $O(\eps\sqrt{\log(1/\eps)})$, it can be improved using boosting. 
The basic idea will be to refocus our attention towards the samples close to the boundary 
between the $+1$ and $-1$ regions of our function. A very convenient way to do this restriction 
is to do rejection sampling in such a way that we end up with another Gaussian 
centered around the separating hyperplane. For this, we need to define 
an appropriate method of rejection sampling.

\begin{definition}
For a unit vector $v\in \R^n$ and real numbers $\theta,\sigma$ with $\sigma<1$, 
define the $(v,\theta,\sigma)$-rejection procedure to be the one that 
given a point $x\in \R^n$ accepts it with probability 
$$\exp(-(\sigma^{-2}-1) (v\cdot x+\theta/(1-\sigma^2))^2/2) \;,$$ 
and rejects it otherwise.
\end{definition}

This definition will be useful to us because of the following property:
\begin{lemma}\label{rejectionLem}
If elements $x$ taken from the standard Gaussian $N(0,I)$ 
are fed into the $(v,\theta,\sigma)$-rejection procedure, 
a point is accepted with probability $\sigma \exp(-\theta^2/(2(1-\sigma^2)))$.
Moreover, the distribution on $x$ conditional on acceptance is that 
of $N(-\theta v, A_{v,\sigma})$, where $A_{v\sigma} = I - (1-\sigma^2) vv^T$
 is the matrix with eigenvalue $\sigma^2$ in the $v$-direction 
 and eigenvalue $1$ in all orthogonal directions.
\end{lemma}
\begin{proof}
First, we note that the distribution of $x$ in directions orthogonal to $v$ is Gaussian 
distributed and independent on both the $v$-component and the rejection probability. 
Therefore, it suffices to consider the one-dimensional problem of a Gaussian 
just along the line parallel to $v$. In this case, the probability that $x=tv$ 
and is accepted by our rejection procedure is to
\begin{align*}
\frac{1}{\sqrt{2\pi}}e^{-t^2/2}\exp(-(\sigma^{-2}-1) (t+\theta/(1-\sigma^2))^2/2) & = \frac{1}{\sqrt{2\pi}}\exp(-(\sigma^{-2}t^2 + 2\sigma^{-2} t\theta +\theta^2/(\sigma^2-\sigma^4))/2)\\
& =\frac{1}{\sqrt{2\pi}}\exp(-(t+\theta)^2/(2\sigma^2))\exp(\theta^2/(2(1-\sigma^2)))\\
& = \sigma \exp(\theta^2/(2(1-\sigma^2)))\left[\frac{1}{\sqrt{2\pi\sigma^2}}\exp(-(t+\theta)^2/(2\sigma^2))\right].
\end{align*}
Since the latter term is the probability density function of $N(-\theta,\sigma^2)$, 
this proves the second statement. This also implies that the integral over $t$ 
must be $\sigma \exp(-\theta^2/(2(1-\sigma^2)))$, which proves the first statement.
\end{proof}

Next, we will want to know what happens when an LTF is passed through this restriction procedure. 
In particular, if we have an LTF $f$ that is nearly $\sgn(v\cdot x +\theta)$ and apply this rejection procedure to the inputs, 
and renormalize the outputs to make them a standard Gaussian, 
we will get samples from another LTF that will tell us about the errors in our original approximation. 
More precisely, we have the following result:
\begin{lemma}\label{restrictionLTFLem}
Let $v\in \R^n$ be a unit vector and let $\theta,\sigma$ be real numbers with $1/2>\sigma$ and $\sigma \theta = O(1)$. 
Let $f(x) = \sgn(u\cdot x +\theta)$ be an LTF with threshold $\theta$, and 
$u=av+bw$ for some unit vector $w\perp v$ and $a, b \in \R$ with $a^2+b^2=1$ and $b=O(\sigma)$. 
Suppose that for some $0<\eps \ll \sigma e^{-\theta^2/2}$, 
we are given an $\eps$-corrupted set of samples from the distribution $(X,f(X))$, 
where $X\sim N(0,I)$, and $(v,\theta,\sigma)$-rejection sample based on the first coordinate. 
Then, the resulting distribution is $O(\eps e^{\theta^2/2}/\sigma)$-close to the distribution 
$(A_{v\sigma}^{1/2} Y-\theta v, g(Y))$, where $Y\sim N(0,I)$ and $g$ is the LTF
$$
g(y) = \sgn((av+bw/\sigma)\cdot y + \theta(1-a)/\sigma) \;.
$$
Furthermore, given $\theta,v,\sigma$, and a $\delta$-approximation to the degree-$1$ Chow parameters of $g$, 
one can obtain an $O(\delta \sigma e^{-\theta^2/2})$-approximation to the Chow parameters of $f$.
\end{lemma}

The last part of this lemma is particularly relevant, 
as if we can $\delta$-approximate the degree-$1$ Chow parameters of $g$ for $\delta=O(\eps e^{-\theta^2/2}/\sigma)$ 
(the error to which we can compute $g$), this allows us to $O(\eps)$-approximate our original $f$. 
Of course, this may still not be possible to do with Corollary \ref{weakLTFCor} alone. 
However, we will only be off by a $\sqrt{\log(1/\delta)}$-factor, 
rather than a $\sqrt{\log(1/\eps)}$ factor. 
This is particularly useful if we can pick $\sigma$ to be very small.

\begin{proof}
If all of our samples were exactly coming from $(X,f(X))$, 
then by Lemma \ref{rejectionLem}, the distribution conditional on acceptance 
would be $(Z,f(Z))$ where $Z$ is distributed as $N(-\theta v, A_{v,\sigma})$. 
Letting $Z=A_{v,\sigma}^{1/2}Y-\theta v$, we have that $Y$ is distributed as the standard normal, 
and our distribution is equivalent to $(A_{v,\sigma}^{1/2}Y-\theta v,g(Y))$, where
\begin{align*}
g(y) & = f(A_{v,\sigma}^{1/2}y-\theta v)\\
& = \sgn((av+bw)\cdot (A_{v,\sigma}^{1/2} y-\theta v) + \theta)\\
& = \sgn(A_{v,\sigma}^{1/2}(av+bw)\cdot y + \theta(1-a))\\
& = \sgn((a\sigma v+bw)\cdot y +\theta(1-a))\\
& = \sgn((av+bw/\sigma)\cdot y + \theta(1-a)/\sigma) \;.
\end{align*}
By Lemma \ref{rejectionLem}, the probability of a sample being accepted
is at least $$\sigma \exp(-\theta^2/(2(1-\sigma^2))) - \eps \gg \sigma e^{-\theta^2/2} \;.$$ 
Therefore, the variation distance between the conditional distribution 
and $(A_{v,\sigma}^{1/2}Y-\theta v,g(Y))$ is at most the distance between 
our original distribution and $(X,f(X))$ 
divided by our probability of accepting, or $O(\eps e^{\theta^2/2}/\sigma)$.

For the last statement, note that $\|av+bw/\sigma\|_2^2 \geq {a^2+b^2} =1$, 
and $\|av+bw/\sigma\|_2^2 \leq a^2 + (b/\sigma)^2 = O(1)$. 
This means that $\|av+bw/\sigma\|_2 = \Theta(1)$. 
Hence, $g$ is an LTF with threshold
$$
\theta(1-a)/\sigma\cdot \Theta(1) = O\left(\frac{\theta(1-a^2)}{(1+a)\sigma} \right) 
= O\left(\frac{\theta b^2}{(1+a)\sigma} \right) = O(1) \;.
$$
Therefore, by Lemma \ref{ChowLTFLem}, the degree-$1$ Chow parameters of $g$ 
are a constant multiple of $av+bw/\sigma$. Thus, if $u$ is a $\delta$-approximation 
of the degree-$1$ Chow parameters, we have that 
$\|u/\|u\|_2 -(av+bw/\sigma)/\|(av+bw/\sigma)\|_2\|_2 = O(\delta)$. 
Taking the component perpendicular to $v$, we find that
$$
\frac{b/\sigma}{\sqrt{a^2+(b/\sigma)^2}} = \frac{\sqrt{|u|_2^2-(v\cdot u)^2}}{|u|_2} + O(\delta) =(C+O(\delta)) \;.
$$
We can then solve for $b$ as
$$
b = \sigma \sqrt{\frac{a^2(C+O(\delta)^2)}{(1-(C+O(\delta))^2)}} \;.
$$
Noting that $C$ is bounded away from $1$, 
this allows us to compute $b$ to error $O(\sigma\delta)$. 
We can then compute $a$ to error $O(\sigma\delta)^2$ as $a=\sqrt{1-b^2}$.

Next, we note that
$$
\left\| (av+bw/\sigma) - \frac{u\sqrt{a^2+(b/\sigma)^2}}{\|u\|_2} \right\|_2 = O(\delta).
$$
Considering the part of $\frac{u\sqrt{a^2+(b/\sigma)^2}}{|u|_2}$ orthogonal to $v$, 
we obtain an $O(\delta)$-approximation of $bw/\sigma$. 
This gives us an $O(\delta\sigma)$-approximation of $bw$, 
and combined with an $O(\delta\sigma)$-approximation to $a$, 
we can obtain an $O(\delta\sigma)$-approximation to $av+bw$, 
the defining vector for $f$. By Lemma~\ref{ChowLTFLem}, this is sufficient 
to obtain an $O(\delta\sigma e^{-\theta^2/2})$-approximation 
to the degree-$1$ Chow parameters $f$. 
This completes our proof.
\end{proof}

This allows us to iteratively improve 
our approximations to the Chow parameters of an LTF.
\begin{lemma}\label{LTFiterationLem}
Let $f$ be an LTF with threshold $\theta$. 
Suppose that we are given $\theta$, 
a $\delta$-approximation to the degree-$1$ Chow parameters of $f$, 
and sample access to an $\eps$-corrupted version of $(G,f(G))$.  
Then, if $\epsilon \ll \delta \ll 1$ and $\delta \theta e^{\theta^2/2} = O(1)$, 
there is an algorithm that takes polynomial time and samples, 
and returns an $O(\eps \sqrt{\log(\delta/\eps)}$-approximation to the degree-$1$ 
Chow parameters of $f$.
\end{lemma}
\begin{proof}
Let $\sigma=\delta e^{\theta^2/2}$ and $v$ be the normalization of our approximation 
of the degree-$1$ Chow parameters of $f$. We note that $\theta \sigma =O(1)$. 
We also note that if $f$ is defined by the unit vector $v'$, 
then the degree-$1$ Chow parameters of $f$ are $2G(\theta)v'$, 
which is within $O(\sigma G(\theta))$ of $2G(\theta)v$. 
Therefore, $\|v-v'\|_2 \leq O(\sigma).$ This means that 
$v'=av+bw$ for some $w\perp v$ and $a^2+b^2=1$ with $b=O(\sigma)$. 
Now taking our samples from $(X,f(X))$ and $(v,\theta,\sigma)$-rejection sampling 
based on the first coordinate, by Lemma \ref{restrictionLTFLem} 
we obtain $O(\eps e^{\theta^2/2}/\sigma)=O(\eps/\delta)$-noisy samples 
to $(A_{v\sigma}^{1/2} Y-\theta v, g(Y))$. Inverting the linear transformation in the first coordinate 
and applying the algorithm from Theorem \ref{thm:robust-chow-gaussian-uniform}, 
we obtain an $O(\eps/\delta \sqrt{\log(\eps/\delta)})$-approximation to 
the degree-$1$ Chow parameters of $g$. Applying Lemma \ref{restrictionLTFLem} again, 
this gives us an $O(\eps/\delta \sqrt{\log(\eps/\delta)}e^{\theta^2/2}/\sigma)$-approximation 
to the degree-$1$ Chow parameters of $g$. 
But this is simply an $O(\eps\sqrt{\log(\delta/\eps)})$-approximation, as desired.
\end{proof}

Iterating this result, we immediately obtain the following corollary:
\begin{corollary}\label{LTFAlgCor}
Let $f$ be an LTF with threshold $\theta$ such that $\theta e^{\theta^2/2} = O(\eps^{-1}/\sqrt{\log(1/\eps)})$. 
Then there is an algorithm that given $\theta$ and sample access to an $\eps$-corrupted version of $(G,f(G))$, 
takes polynomial time and samples and returns and $O(\eps)$-approximation 
to the degree-$1$ Chow parameters of $f$.
\end{corollary}
\begin{proof}
The algorithm is as follows:
\begin{enumerate}
\item Using Theorem \ref{thm:robust-chow-gaussian-uniform}, we obtain a $\delta_0=O(\eps\sqrt{\log(1/\eps)})$-approximation to the degree-$1$ Chow parameters of $f$.
\item Let $i=0$.
\item \label{iterateStep} Let $i\leftarrow i+1$, and use Lemma \ref{LTFiterationLem} to obtain a $\delta_i = C(\eps\sqrt{\log(\delta_{i-1}/\eps)})$-approximation to the degree-$1$ Chow parameters of $f$, for some sufficiently large $C$.
\item If $\delta_i < \delta_{i-1}/2$, return to Step \ref{iterateStep}.
\item \label{endLoopStep} Letting $v$ be the normalization of the approximation 
to the degree-$1$ Chow parameters of $f$, return the function $\sgn(v\cdot x + \theta)$.
\end{enumerate}
To prove correctness, note that $\delta_i >\eps$ for all $i$, 
and therefore $\delta_i\theta e^{\theta^2/2} = O(1)$ for all $i$, 
so the hypotheses of Lemma \ref{LTFiterationLem} 
are always satisfied in Step \ref{iterateStep}. 
Next note that $\delta_i/\eps = C\sqrt{\log(\delta_i/\eps)}$, 
so the $\delta_i$ are decreasing and always shrinking by a factor of at least $2$, 
unless $\delta_{i-1}=O(\eps)$. Therefore, we reach Step \ref{endLoopStep} 
in at most $\log(\delta_0/\eps)$ iterations, 
and when we do $\delta_i=O(\eps)$. This completes the proof.
\end{proof}

Unfortunately, this algorithm only works when $\theta e^{\theta^2/2} = O(\eps^{-1}/\sqrt{\log(1/\eps)})$, 
while we would need to deal with $\theta e^{\theta^2/2}$ as large as $\eps^{-1}$ to make our algorithm work in general. 
This is for somewhat technical reasons. Essentially, if we have very extreme thresholds, 
our rejection sampling procedure will fail. This happens because the Gaussian we need after restriction is too wide. 
This will mean that we need a reasonable chance of selecting points even further than $\theta$ 
from the origin in the $v$-direction, and this will in turn force our acceptance probability to be too small. 
To correct this issue, we will want to restrict to an even narrower Gaussian. 
Of course, this will make our acceptance probability even smaller, 
and thus the fraction of accepted points that are in error will become much larger. 
However, we will also allow ourselves to vary the exact threshold 
at which we perform our cutoff, this will mean that on average 
the fraction of accepted points that are in error will not be too big.

We can use these ideas to prove an improved version 
of Lemma \ref{LTFiterationLem} that gets around the $\theta e^{\theta^2/2} = O(\eps^{-1})$ 
condition, in exchange for producing a poly-logarithmic number of outputs.

\begin{proposition}\label{LTFiterationProp}
Suppose that we are given real numbers $\theta,\eps,\delta$ 
with $O(\eps\sqrt{\log(1/\eps)}) > \delta > \eps > 0$, 
and $1-|\erf(\theta)|$ is at least a sufficient multiple of $\eps$.

Let $f$ be an LTF with threshold $\theta$. There is an algorithm that given 
$\eps,\delta,\theta$, a $\delta$-approximation, $u$, to the degree-$1$ Chow parameters of $f$, 
and sample access to an $\eps$-corrupted version of $(G,f(G))$, 
takes polynomial time and returns a vector that, with probability 
at least $1/\polylog(\eps)$, is a $(\delta/2+O(\eps))$-approximation 
to the degree-$1$ Chow parameters of $f$.
\end{proposition}
\begin{proof}
Note that we may assume for a sufficiently large constant $C$ 
that $1/(C\eps)> \theta e^{\theta^2/2} > C\sqrt{\log(1/\eps)}/\eps$, 
as if the first inequality fails, we have $1-|\erf(f)| < O(\eps)$, 
and if the second fails, we may use Lemma \ref{LTFiterationLem}. 
This implies that $\eps \log(1/\eps)/C \gg G(\theta) \gg C\eps \sqrt{\log(1/\eps)}$. 
We assume this throughout the following.

Let $v=u/\|u\|_2.$ We note that $v$ is a $\delta / G(\theta)$-approximation 
to the defining vector of $f$. We can write this defining vector uniquely 
as $av+bw$ for non-negative real numbers $a,b$ with $a^2+b^2=1$, 
and a vector $w\perp v$. We note that $b=O(\delta / G(\theta)) = O(\delta/ (C\eps\sqrt{\log(1/\eps)}))=O(1/C)$. 
Thus, we may assume that $b$ is less than a sufficiently small constant. 
However, rounding $b$ to the nearest multiple of $1/\log(1/\eps)$ introduces 
a variation distance error of at most $\eps$, and an $O(\eps)$ error in the degree-$1$ Chow parameters. 
Therefore, up to introducing another $O(\eps)$ error in our sampling, 
we may assume that $b$ is a multiple of $1/\log(1/\eps)$. 
Guessing the value of $b$, we note that we are correct with probability $1/\log(1/\eps)$. 
The remainder of this algorithm is conditional on this correctness. 
Thus, henceforth, we will assume that the algorithm 
knows the value of $b$, and hence also knows the value of $a$.


Next, pick a random threshold $s\in [a\theta,a\theta+b]$. 
This will be the threshold that we will try to restrict to.

We will then apply the $(v,s,\sigma)$-rejection procedure with $\sigma=1/\theta$ 
to samples from our noisy version of $(G,f(G))$ rejecting based on the first coordinate. 
If there were no errors, our acceptance probability would be $\sigma e^{-s^2/(2(1-\sigma^2))} = \Omega(\sigma e^{-s^2/2}).$ However, it will be important to know that it is impossible to have our errors 
be too likely to be accepted by this procedure. Now it is possible that, for certain values of $s$, 
we will accept too many errors. However, we wish to show that on average it is not too many. 
In particular, for a point $x$ we consider $\E_s[\Pr(x\textrm{ is accepted})].$ In particular,
\begin{align*}
\E_s[\Pr(x\textrm{ is accepted})] & = \frac{1}{b}\int_{a\theta}^{a\theta+b} \exp(-(\sigma^{-2}-1)(v\cdot x+s/(1-\sigma^2))^2/2)ds\\
& \leq  \frac{1}{b}\int_{-\infty}^{\infty} \exp(-(\sigma^{-2}-1)(v\cdot x+s/(1-\sigma^2))^2/2)ds\\
& = O(\sigma/b) \;.
\end{align*}
This means that, for most $s$, the sum of the fraction of samples that are 
either bad and accepted or would have been accepted if they were not corrupted 
is $O(\eps\sigma/b)$. For such $s$, the fraction of accepted samples that come 
from corrupted samples is at most $O(\eps e^{s^2/2}/b)$. 
We assume in the following that the algorithm found such an $s$.

We will now need to mimic the latter half of Lemma \ref{rejectionLem}. 
In particular, were there no corruptions, the accepted samples 
would be from the distribution $(Z,f(Z))$ with $Z\sim N(-sv,A_{v,\sigma})$, 
though as it stands we have instead an $\eta:=O(\eps e^{s^2/2}/b)$-noisy version of this. 
Letting $Z=A_{v,\sigma}^{1/2}Y-sv$, we find that $Y$ is distributed 
as a standard Gaussian and our distribution 
is close to $(A_{v,\sigma}^{1/2}Y-sv,g(Y))$, where $g$ is the LTF
$$
g(y) = \sgn((av+bw)\cdot (A_{v,\sigma}^{1/2}y-sv) + \theta)  = \sgn((a\sigma v+bw)\cdot y + (\theta-as)) \;.
$$
Note that the threshold of $g$ is
$$
\frac{\theta-as}{\sqrt{(a\sigma)^2 + b^2}},
$$
which has absolute value at most $(\theta -as)/b$.

Now employing Theorem \ref{thm:robust-chow-gaussian-uniform}, 
we can learn the degree-$1$ Chow parameters of $g$ to error $O(\eps e^{s^2/2}/b \sqrt{\log(1/\eps)})$. 
By Lemma \ref{ChowLTFLem}, this allows us to learn the defining vector of $g$ to error
\begin{align*}
O(\eps e^{s^2/2}/b \sqrt{\log(1/\eta)}e^{(\theta -as)^2/2b^2}) & = O(\eps \sqrt{\log(1/\eta)} ) \exp( (s^2+(\theta-as)^2/b^2)/2)\\
& = O(\eps \sqrt{\log(1/\eta)}/b )\exp( (s^2+\theta^2/b^2-2as\theta/b^2+a^2s^2/b^2)/2)\\
& = O(\eps \sqrt{\log(1/\eta)}/b )\exp( (\theta^2+(s/b - \theta a /b)^2)/2)\\
& = O(\eps \sqrt{\log(1/\eta)}/b )e^{\theta^2/2} \;.
\end{align*}
On the other hand, this defining vector is a known constant multiple of $w+a\sigma v/b$. 
Therefore, we can learn $w$ to error $O(\eps \sqrt{\log(1/\eta)}e^{\theta^2/2}/b )$, 
and thus learn the degree-$1$ Chow parameters of $f$ to error $O(\eps \sqrt{\log(1/\eta)})$.

We have that 
\begin{align*}
1/\eta & = O((\delta/\eps) e^{-s^2/2}/G(\theta))\\
& = O((\delta/\eps)) e^{(\theta^2-s^2)/2}\\
& \leq O((\delta/\eps)) \exp(\theta^2 /2 -a^2\theta^2/2)\\
& = O((\delta/\eps)) \exp((b\theta)^2/2) \;.
\end{align*}
Note that $b\theta$ itself cannot be too big. In particular, we have that
\begin{align*}
b\theta & = O(\delta \theta/ G(\theta)) = O(\delta /C \eps).
\end{align*}
Thus,
$$
\sqrt{\log(1/\eta)} = O(1+\sqrt{O(\log(\delta/\eps)+(\delta/\eps)^2/C^2)}) = O(1+\sqrt{O((\delta/\eps)^2/C^2)}) \leq O(1)+\delta/(2\eps) \;.
$$
Therefore, we learn the defining vector of $f$ to error $O(\eps)+\delta/2$.

The final algorithm is as follows:
\begin{enumerate}
\item If $\theta e^{\theta^2/2} < 1/\eps$, use Lemma \ref{LTFiterationLem}.
\item Let $b$ be a random multiple of $1/\log(1/\eps)$ between $0$ and $1$ and let $a$ be the positive real number so that $a^2+b^2=1$.
\item Let $s$ be a uniform random element of $[a\theta,a\theta+b]$.
\item Apply the $(v,s,\sigma)$-rejection procedure with $\sigma=1/\theta$ to our sample set, 
treating the accepted samples as $(A_{v,\sigma}^{1/2}Y-sv,g(Y))$.
\item Assuming that this is an $\eta$-noisy copy of an LTF $g$ with $\eta=O(\eps e^{s^2/2}/b)$, 
use Theorem \ref{thm:robust-chow-gaussian-uniform} to learn the degree-$1$ Chow parameters of $g$ 
to error $O(\eta\sqrt{\log(1/\eta)})$, call these $x$.
\item Let $w$ be the solution to $x/\|x\|_2 = (a\sigma v+bw)/\sqrt{(a\sigma)^2+b^2}$.
\item Return $2G(\theta)(av+bw)$.
\end{enumerate}
\end{proof}

We can now iterate Proposition \ref{LTFiterationProp} to obtain the following:
\begin{corollary}
Suppose that we are given real numbers $\theta,\eps$ with $ \eps > 0$, 
and $1-|\erf(\theta)|$ is at least a sufficient multiple of $\eps$.

Let $f$ be an LTF with threshold $\theta$. 
There is an algorithm that given $\eps,\theta$, and sample access to an $\eps$-corrupted version of $(G,f(G))$, 
takes polynomial time and returns a vector that, with probability at least $\log(1/\eps)^{-O(\log\log(1/\eps))}$, 
is an $O(\eps)$-approximation to the degree-$1$ Chow parameters of $f$.
\end{corollary}
\begin{proof}
The algorithm is as follows:
\begin{enumerate}
\item Let $C$ be a sufficiently large constant.
\item Run Theorem \ref{thm:robust-chow-gaussian-uniform} to compute $u_0$, a $\delta_0:=C\eps\sqrt{\log(1/\eps)}$-approximation of the degree-$1$ Chow parameters.
\item For $i=1$ to $C\log\log(1/\eps)$
\begin{enumerate}
\item Let $u_i$ be the output of the algorithm from Proposition \ref{LTFiterationProp} run on our samples with inputs $\eps,\theta,\delta_{i-1},u_{i-1}$.
\item Let $\delta_i = \delta_{i-1}/2+C\eps.$
\end{enumerate}
\item Return $u_i$.
\end{enumerate}

Note that by induction on $i$, we have that, with probability at least $\polylog(1/\eps)^i$, $u_i$ 
is a $\delta_i$-approximation of the degree-$1$ Chow parameters of $f$. 
Note also that $\delta_i = \delta_0/2^i +O(C\eps)$. 
Therefore, with probability at least $\log(1/\eps)^{-O(\log\log(1/\eps))}$, 
the final returned value is an $O(\eps)$ approximation to the degree-$1$ Chow parameters of $f$.
\end{proof}

We may now prove Theorem \ref{LTFThrm}.
\begin{proof}
First, compute an $\eps$ approximation of $\E[f]$, and pick a $\theta$ so that $\erf(\theta)=\E[f]+O(\eps)$. 
Up to increasing the error by a constant factor, we may assume that $f$ has threshold $\theta$. 
If $1-|\erf(\theta)|$ is less than a constant multiple of $\eps$, 
we may return the constant function $f=\sgn(\theta)$. 
Otherwise, running the above corollary $\log(1/\eps)^{O(\log\log(1/\eps))}$ times, 
we come up with $\log(1/\eps)^{O(\log\log(1/\eps))}$ different hypotheses for $f$ with the promise that 
at least one of them is within $O(\eps)$ with probability at least $9/10$. 
Running a standard hypothesis testing procedure (see, e.g.,~\cite{DDS12stoc,DDS15}) over these possibilities, 
we obtain our final answer.
\end{proof}

\section{Robust Learning of Intersections of LTFs under the Gaussian Distribution} \label{sec:intersections}

In this section, we prove our algorithmic result for intersections of LTFs.
Specifically, we show the following theorem, a detailed version 
of Theorem~\ref{thm:intersections-informal}:

\begin{theorem} [Learning Intersections of LTFs with Nasty Noise] \label{thm:intersections}
There exists an algorithm that given $k \in \Z_+,\eps>0$, 
and sample access to an $\eps$-corrupted set of labeled samples from $f: \R^n \to \{\pm 1\}$, 
the indicator function of an intersection of $k$ LTFs, with respect to the Gaussian distribution $N(0, I)$, 
draws $\poly(n, k, 1/\eps)$ samples and takes $\poly(n, k, 1/\eps)+(k/\eps)^{O(k^2)}$
time to compute an intersection of $k$ LTFs hypothesis $h$ that, with probability at least $9/10$, satisfies
$\Pr_{X \sim N(0, I)} [h(X) \neq f(X)] \leq O(\eps^{1/11}k^{4/11}\log^{3/11}(k/\eps))$.
\end{theorem}

Our algorithm makes essential use of the following structural result, a detailed version of 
Theorem~\ref{thm:structural-ltfs-informal}, whose proof is deferred to the following subsection:

\begin{proposition} [Robust Inverse Independence for Intersections of LTFs] \label{halfspaceStrcutureProp}
Let $f:\R^n\rightarrow \{0, 1\}$ be the indicator function of an intersection of $k$ LTFs. 
Suppose that there is some unit vector $v$ so that for $p$ any degree at most $2$ 
polynomial with $\E[p(G)]=0$ and $\E[p^2(G)]=1$ 
we have that $|\E[f(G)p(v\cdot G)]|<\delta$. Then if $G$ and $G'$ are Gaussians 
that are correlated to be the same in the directions orthogonal to $v$ 
and independent in the $v$-direction, then $\E[|f(G)-f(G')|] \leq O(\delta^{1/11}k^{4/11}\log^{2/11}(k/\delta))$.
\end{proposition}

Given the above proposition, the algorithm to establish Theorem~\ref{thm:intersections}
is quite simple.

\subsection{Proof of Theorem~\ref{thm:intersections}}
The idea of the algorithm is quite simple. 
Using the algorithm from Theorem \ref{thm:robust-chow-gaussian-uniform}, 
we compute approximations to the degree-$1$ and degree-$2$ Chow parameters of $f$. 
Note that these Chow parameters allow us to approximate $\E[f(G)p(G)]$ 
for any degree at most $2$ polynomial $p$. 
Using Proposition \ref{halfspaceStrcutureProp}, this allows us to identify a low-dimensional subspace $V$, 
so that $f(x)$ is close in variation distance to $g(\pi_V(x))$, 
for $g$ the indicator function of an intersection of $k$ LTFs. 
However, since $g$ is defined on a low-dimensional space, 
we can easily determine a sufficient $g$ using standard cover arguments. 
The algorithm is as follows:
\begin{enumerate}
\item Using the algorithm from Theorem \ref{thm:robust-chow-gaussian-uniform} to compute $v$ and $\Sigma$, 
which are $O(\eps\log(1/\eps))$-approximations to the degree-$1$ and degree-$2$ Chow parameters of $f$, respectively.
\item Let $V$ be the subspace spanned by $v$ and the eigenvectors of $\Sigma$ corresponding to the $k$ largest eigenvalues.
\item Let $\delta$ be a sufficiently large multiple of $\eps^{1/11}k^{4/11}\log^{3/11}(k/\eps)$.
\item Let $\mathcal{C}$ be a $\delta$-cover of the set of intersections of $k$ LTFs on $V$.
\item Using a standard hypothesis testing routine (tournament), 
find an element $g$ of $\mathcal{C}$ so that $(G,g(\pi_V(G)))$ is $\delta$-close to $(G,f(G))$.
\item Return $h(x)=g(\pi_V(x))$.
\end{enumerate}
To analyze this algorithm, we would first like to use Proposition \ref{halfspaceStrcutureProp} 
to show that $f$ is $\delta$-close to being a function of the form $g(\pi_V(x))$, 
for $g$ some intersection of LTFs. To do this we need to show that, for $u$ of unit norm orthogonal to $V$, 
for any normalized, mean $0$ polynomial $p$ it holds that $\E[f(G)p(u\cdot G)]$ is small. 
Note that $p(u\cdot G)$ is a linear combination of $u\cdot G$ and $(u\cdot G)^2-1$ with $O(1)$ coefficients. 
Letting $v_0$ and $\Sigma_0$ be the true degree-$1$ and degree-$2$ Chow parameters of $f$, 
we have that $\E[f(G)(u\cdot G)] = u\cdot v_0$ and $\E[f(G)((u\cdot G)^2-1)] = u^T \Sigma_0 u.$ 
We need to show that each of these are small.

Since $u\perp V$, we have $u\cdot v=0$ and thus that $u\cdot v_0 = u\cdot v+u\cdot (v_0-v) = O(\eps\log(1/\eps))$.

The other term is slightly more challenging. 
We similarly have that $u^T\Sigma_0 u = u^T \Sigma u +O(\eps\log(1/\eps)$. 
Since $u$ is orthogonal to the top $k$ eigenvectors of $\Sigma$, 
it must be the case that $u^T \Sigma u \leq \lambda_{k+1}$, the $(k+1)^{st}$ eigenvalue of $\Sigma$. 
We need to show that this is small. To do so, we will show that for any subspace $W$ of dimension $k+1$, 
there exists a unit vector $w\in W$ with $w^T\Sigma w$ small. 
For this, we note that since $\Sigma_0$ is rank $k$, there exists such a $w$ in the kernel of $\Sigma_0$. 
For this $w$, we thus have that $w^T\Sigma w = w^T(\Sigma-\Sigma_0)w=O(\eps\log(1/\eps))$. 
Therefore, $\lambda_{k+1} = O(\eps\log(1/\eps))$, and thus, $u^T\Sigma_0 u= O(\eps\log(1/\eps))$.

Now applying Proposition \ref{halfspaceStrcutureProp}, 
we know that $f$ is $\delta$-close to $g(\pi_V(x))$, for $g$ some intersection of LTFs.

The remaining analysis is straightforward. We can easily produce a $\delta$-cover of size $O(k/\delta)^{k(k+1)}$, 
since we only need an intersection of $k$ LTFs in $(k+1)$-dimensions. 
We know by the above that some $g$ should cause the distributions in question to be close, 
and the hypothesis testing procedure will find it with an appropriate number of samples. 
This completes the proof. \qed

\subsection{Proof of Proposition~\ref{halfspaceStrcutureProp}}

We first separate out the coordinates of $f$ into those in the $v$-direction, 
and those in orthogonal directions. We let $f(x,y)$, where $x\in \R$ and $y\in \R^n$, $y\cdot v=0$ denote $f(xv+y)$.

We proceed to prove the contrapositive. 
Assume that $\E[|f(G)-f(G')|] = \E[|f(x,G)-f(x',G)|] > \eta$ 
and show that there is some $p$ with $|\E[f(G)p(v\cdot G)]|$ large. 
Our basic idea will be to consider the projection of $f$ onto the line defined by $v$. 
Namely, let
$$
g(x) = \E[f(x,G)] \;.
$$
We note that $g$ is the projection of a log-concave function, 
and therefore, is log-concave. 
In particular, this means that $g$ is unimodal. 
If we can show that $g$ is not too close to being constant, we will obtain our result.

To do this, we note that if $\E[|f(x,G)-f(x',G)|]$ is large, 
there must be some pair $x$ and $y$ so that $f(x,z)$ and $f(y,z)$ 
are far apart as functions of $z$. We claim that this will imply that 
$g(x), g(y)$, and $g(z)$ cannot be close for all $z$ 
between $x$ and $y$. In particular, we show:
\begin{lemma}
Suppose that for some $x, y$ that $\|f(x,w)-f(y,w)\|_1 > \gamma$ 
(where the $L_1$-norm is taken over $w$ being assigned Gaussian values). 
Then, there exists a $z$ between $x$ and $y$ 
so that some pair of $g(x), g(y)$, and $g(z)$ differ by at least
$$\Omega\left(\frac{\gamma^5}{k^4\log^2(2k/\gamma )}\right) \;.$$
\end{lemma}
\begin{proof}
We let $z=\alpha x + (1-\alpha)y$ for some $\alpha$ to be chosen later. 
Because projections of log-concave functions are log-concave, 
$g(z)$ must be at least $g(x)^{\alpha}g(y)^{1-\alpha}\geq \min(g(x),g(y))$. 
Our basic plan will be to show that this cannot be tight.

Let $f_a(w)=f(a,w)$. We may assume without loss of generality 
that $\E[f_x(G)] \leq \E[f_y(z)]$. This means that $\Pr(f_y(G)=1, f_x(G)=0) \geq \gamma/2.$ 
Note that since $f_x$ is the indicator function of an intersection of $k$ LTFs, 
the set on which $f_x(w)=0$ is a union of $k$ LTFs. 
Therefore, there must be a halfspace $H$ on which $f_x$ is 0, 
and so that $\Pr(f_y(G)=1,G\in H) \geq \gamma/(2k)$. 
Let $H$ be the halfspace $u\cdot z \geq s$ for some unit vector $u$. 
Let $h(a,b)$ be the projection of $f_a$ onto the $u$-direction. 
Namely, $h(a,b)=\E[f_a(G)|u\cdot G = b]$. 
Note that $h$ is a $2$-variable log-concave function and that
$$
g(a) = \int_{-\infty}^\infty \frac{1}{\sqrt{2\pi}} e^{-t^2/2} h(a,t) dt \;.
$$
Also note that being a projection of $f$, 
we have that $h$ takes values in $[0,1]$. 
We also set $H(a,t) = \frac{1}{\sqrt{2\pi}} e^{-t^2/2}h(a,t)$.

Note also that $H(x,b)=0$ for $b\geq s$ and $\int_s^\infty H(y,t) \geq \gamma/(2k).$

Let $H'(t) = \sup_{\alpha a+(1-\alpha)b=t} H(x,a)^{\alpha}H(y,b)^{1-\alpha}$. 
Note by the log-concavity of $H$ that $H(z,t) \geq H'(t)$. 
Furthermore, by standard results we have that
$$
\int_{-\infty}^\infty H'(t) dt \geq \left(\int_{-\infty}^\infty H(x,t) dt \right)^\alpha\left(\int_{-\infty}^\infty H(y,t) dt \right)^{1-\alpha}= g(x)^{\alpha}g(y)^{1-\alpha} \;.
$$
Our goal will be to show that
$$
\int_{-\infty}^\infty H(z,t) dt \textrm{ is substantially larger than } \int_{-\infty}^\infty H'(t) dt \;.
$$
The basic idea of the proof is that if $H'(t) = H(x,a)^\alpha H(y,b)^{1-\alpha}$, for some $\alpha a+(1-\alpha) b=t$, we have that
\begin{align*}
H(z,t) & = \frac{1}{\sqrt{2\pi}} e^{-t^2/2} h(z,t)\\
& \geq \frac{1}{\sqrt{2\pi}} e^{-t^2/2} h(x,a)^{\alpha}h(y,b)^{1-\alpha} \\
& \geq e^{-t^2/2+(\alpha a^2+(1-\alpha) b^2)/2} H(x,a)^\alpha H(y,b)^{1-\alpha}\\
& = e^{\alpha(1-\alpha)(a-b)^2/2} H'(t)\\
& = e^{\alpha(t-a)^2/(2(1-\alpha)}H'(t) \;.
\end{align*}
This is particularly relevant when $t \geq s$, as $t-a \geq t-s$. 
In particular, for some parameter $\beta$ (to be chosen later), we have that
$$
\int_{-\infty}^\infty H(z,t) dt - \int_{-\infty}^\infty H'(t) dt \geq (\alpha\beta^2/2) \int_{s+\beta}^\infty H'(t)dt \;.
$$
Note that since $H(x,t)$ integrates to $g(x)$ 
and since it is bounded by the Gaussian pdf, 
we have that the integral of $H(x,t)$ for $|t| < 2\log(2/g(x))$ is at least $g(x)/2$. 
Therefore, there is some $a_0$ with $|a_0|\leq 2\log(2/g(x))$ 
so that $H(x,a_0) \geq g(x)/8\log(2/g(x)).$ Therefore, we have that
$$
H'(t) \gg H(x,a_0) H(y,(t-\alpha a_0)/(1-\alpha)) \geq g(x)/8\log(2/g(x))H(y,(t-\alpha a_0)/(1-\alpha)) \;.
$$
Note that $\int_s^{s+\gamma/(4k)} H(y,t)dt \leq \gamma/(4k).$ 
Therefore, $\int_{s+\gamma/(4k)}^\infty H(y,t)dt \geq \gamma/(4k).$ 
Choose $\alpha$ so that $\alpha a_0 + (1-\alpha)(s+\gamma/(4k)) = s+\gamma/(8k)$. 
In other words, $\alpha(s+\gamma/(4k)-a_0) = \gamma/(8k)$, 
so $\alpha \gg \gamma/(k\log(2k/(g(x)\gamma))).$ 
Furthermore, since $a_0 \leq s$, $\alpha \leq 1/2$. Let $\beta = \gamma/(8k)$. 
We have that
\begin{align*}
\int_{s+\beta}^\infty H'(t)dt & \geq \int_{s+\beta}^\infty g(x)/8\log(2/g(x))H(y,(t-\alpha a_0)/(1-\alpha)) \\ 
& \gg g(x)/\log(2/g(x)) \int_{s+\gamma/(4k)}^\infty H(y,t)dt \gg \gamma g(x)/(k\log(2/g(x))) \;.
\end{align*}
Therefore, we have that
$$
g(z)-\min(g(x),g(y)) \geq \int_{-\infty}^\infty H(z,t) dt - \int_{-\infty}^\infty H'(t) dt \gg g(x)\gamma^4/(k^4\log^2(2k/(\gamma g(x)))) \;.
$$
Now if $g(x) \geq \gamma/3$, we are done. 
Otherwise, we must have $g(y) \geq 2\gamma/3$, 
and we can already attain a difference of $\gamma/3$ 
between $g(x)$ and $g(y)$. This completes the proof.
\end{proof}

If we have that $\E[|f(x,G)-f(x',G)|] > \eta$, 
then there must be some $x$ and $y$ 
not in the $\eta/4$-tails of the Gaussian distribution 
so that $\|f(x,w)-f(y,w)\|_1 \geq \eta/4$ and $|x-y|\gg \eta$. 
The above lemma implies that there is some (potentially different) 
pair $x$ and $y$ not in the $\eta/3$-tails of the distribution 
so that $|g(x)-g(y)| \gg \eta^5 k^{-4} \log^{-2}(2k/\eta).$ 
We claim that this is enough to find a polynomial $p$.

\begin{lemma}
Suppose that $g:\R\rightarrow [0,1]$ is a log-concave function and $\alpha>0$. 
Suppose that there exist $x, y$ not in the $\alpha$-tails of a Gaussian distribution 
with $|x-y|>\alpha$ so that $|g(x)-g(y)|>\beta$. 
Then there exists a degree-$2$ polynomial $p$ 
with $\E[p(G)]=0$ and $\E[p^2(G)]=1$ so that $|\E[g(G)p(G)]| \gg \beta\alpha^6$.
\end{lemma}
\begin{proof}
First, note that polynomials $p$ with expectation $0$ are linear combinations of $x^2-1$ and $x$. 
Therefore, their quadratic term and their unit term are negatives of each other, 
and therefore the product of their roots is $-1$. Let $t$ be the smallest number so that $g^{-1}((t,1])$ 
is contained in an interval where the product of the endpoints is at least $-1$. 
There exists an interval $I=[-1/a,a]$ so that $g$ is at least $t$ on the interior of $I$ 
and at most $t$ outside of $I$. We let $p$ be the unique degree-$2$ polynomial 
with $\E[p(G)]=0$ and $\E[p^2(G)]=1$, 
so that $p$ has roots $a$ and $-1/a$ and positive leading term.

It is clear that $\E[g(G)p(G)] >0$ since it is $\E[(g(G)-t)p(G)]$ and $(g(x)-t)(p(x))$ 
is everywhere non-negative. It only remains make this claim effective.

First, we proceed by improving the separation between $x$ and $y$. 
Without loss of generality, assume that $g(x) > g(y)$ and $x>y$. 
Let $I_x = [(x+y)/2,x].$ By log-concavity, we have that $g$ is at least 
$g(y)+\Omega(\beta)$ on $I_x$. Let $I_y = [y-\alpha,y]$. 
We have that $g$ is at most $g(y)$ on $I_y$. 
Furthermore, note that the Gaussian mass of each of $I_x$ and $I_y$ is at least $\Omega(\alpha^2)$.

We note that for one of the two intervals $I_x$ or $I_y$ 
the values taken by $g$ on this interval are always at least $\Omega(\beta)$ far away from $t$. 
Therefore, there exists an interval $J$ with Gaussian mass at least $\Omega(\alpha^2)$ 
so that $|g(x)-t|=\Omega(\beta)$ on $J$. It will now suffice to bound from below 
the expectation of $(g(G)-t)p(G)\mathbb{1}_J(G)$. 
This is at least $\Omega(\beta)\E[|p(G)|\mathbb{1}_J(G)].$ 
However, since $p$ is a normalized degree-$2$ polynomial, 
by standard anti-concentration bounds~\cite{CW:01} we have
that the probability that $|p(G)| < c \alpha^4$ 
is less than half the Gaussian mass of $J$, 
when $c$ is sufficiently small. 
Therefore, this expectation is at least $\Omega(\beta\alpha^6)$. 
This completes our proof.
\end{proof}

Applying this lemma, immediately gives a polynomial $p$ 
with $\E[p(G)]=0$ and $\E[p^2(G)]=1$, 
so that $|\E[g(G)p(G)]| \gg \eta^11 k^{-4} \log^{-2}(2k/\eta)$. 
Letting $q$ be the multivariate polynomial defined by 
$q(x)=\pm p(v\cdot x)$, we find that $q$ is a mean $0$, 
variance $1$ polynomial with $\E[f(G)q(G)] \gg\eta^11 k^{-4} \log^{-2}(2k/\eta)$. 
Thus, if $\E[|f(G)-f(G')|] > \eta$, there is a $p$ with $\E[f(G)p(G)] \gg\eta^11 k^{-4} \log^{-2}(2k/\eta)$. 
Equivalently, if there is no such polynomial $p$ with $\E[f(G)p(G)] \geq \delta$, 
it must be the case that $\E[|f(G)-f(G')|] = O(\delta^{1/11}k^{4/11}\log^{2/11}(k/\delta))$, as desired. 
\qed

Improving on this, we obtain the following corollary:

\begin{corollary}
Let $f:\R^n\rightarrow \{0, 1\}$ be the indicator function of an intersection of $k$ LTFs. 
Suppose that there exists a vector space $V$ so that for all vectors $v\perp V$ 
and for $p$ any degree at most $2$ polynomial with $\E[p(G)]=0$ and $\E[p^2(G)]=1$ 
we have that $|\E[f(G)p(v\cdot G)]|<\delta$. Then, there exists a function $g:\R^n\rightarrow \{0, 1\}$, 
also the indicator function of the intersection of at most $k$ LTFs, 
so that, for all $x$, $g(x)=g(\pi_V(x))$, 
and so that $\|f-g\|_1 = O(\delta^{1/11}k^{15/11}\log^{2/11}(k/\delta)).$
\end{corollary}
\begin{proof}
Let $W$ be the span of the vectors defining the LTFs defining $f$. 
Note that we already have that $f(x)=f(\pi_{V\oplus W}(x))$, 
therefore, we lose nothing by restricting our problem to $V\oplus W$. 
Thus, we may assume that $n\leq \dim(V)+k$. Without loss of generality, 
we may assume that $V$ is the span of the first $m$ coordinates. 
Letting $g_1,\ldots,g_n,g_1',\ldots,g_n'$ be independent, one-variable Gaussians, 
we have by our proposition that
\begin{align*}
& \E[|f(g_1,\ldots,g_m,g_{m+1},\ldots,g_n)- f(g_1,\ldots,g_m,g_{m+1}',\ldots,g_n')|]\\
= & \sum_{i=m+1}^n \E[|f(g_1,\ldots g_i,g_{i+1}',\ldots, g_n')- f(g_1,\ldots g_{i-1},g_{i}',\ldots, g_n')|]\\
= & O(\delta^{1/11}k^{15/11}\log^{2/11}(k/\delta)) \;.
\end{align*}
Therefore, writing $f(x)=f(x_V,x_W)$, 
where $x_V$ is the first $m$ coordinates and $x_W$ the remaining coordinates, we have that
$$
\E[|f(G,G_1)-f(G,G_2)|] = O(\delta^{1/11}k^{15/11}\log^{2/11}(k/\delta)) \;.
$$
This implies that there should be a fixed value of $G_2 = X$ 
so that the expectation over the remaining variables is
$$
\E[|f(G)-f(\pi_V(G),X)|]=O(\delta^{1/11}k^{15/11}\log^{2/11}(k/\delta)) \;.
$$
Taking $g(x)=f(\pi_V(x),X)$, yields our result. 
\end{proof}

\bibliographystyle{alpha}

\bibliography{allrefs}

\newcommand{\etalchar}[1]{$^{#1}$}
\begin{thebibliography}{DKK{\etalchar{+}}17b}

\bibitem[ABL17]{ABL17}
P.~Awasthi, M.~F. Balcan, and P.~M. Long.
\newblock The power of localization for efficiently learning linear separators
  with noise.
\newblock {\em J. {ACM}}, 63(6):50:1--50:27, 2017.

\bibitem[Bau91]{Baum:91}
E.~Baum.
\newblock A polynomial time algorithm that learns two hidden unit nets.
\newblock {\em Neural Computation}, 2:510--522, 1991.

\bibitem[BEK02]{BEK:02}
N.~Bshouty, N.~Eiron, and E.~Kushilevitz.
\newblock {PAC Learning with Nasty Noise}.
\newblock {\em Theoretical Computer Science}, 288(2):255--275, 2002.

\bibitem[Bru90]{Bruck:90}
J.~Bruck.
\newblock Harmonic analysis of polynomial threshold functions.
\newblock {\em SIAM Journal on Discrete Mathematics}, 3(2):168--177, 1990.

\bibitem[Cho61]{Chow:61short}
C.K. Chow.
\newblock {On the characterization of threshold functions}.
\newblock In {\em Proc. 2nd FOCS}, pages 34--38, 1961.

\bibitem[CW01]{CW:01}
A.~Carbery and J.~Wright.
\newblock {Distributional and $L^q$ norm inequalities for polynomials over
  convex bodies in $R^n$}.
\newblock {\em Mathematical Research Letters}, 8(3):233--248, 2001.

\bibitem[Dan15]{Daniely15}
A.~Daniely.
\newblock A {PTAS} for agnostically learning halfspaces.
\newblock In {\em Proceedings of The 28th Conference on Learning Theory, {COLT}
  2015}, pages 484--502, 2015.

\bibitem[Dan16]{Daniely16}
A.~Daniely.
\newblock Complexity theoretic limitations on learning halfspaces.
\newblock In {\em Proceedings of the 48th Annual Symposium on Theory of
  Computing, {STOC} 2016}, pages 105--117, 2016.

\bibitem[DDFS14]{DeDFS14}
A.~De, I.~Diakonikolas, V.~Feldman, and R.~A. Servedio.
\newblock Nearly optimal solutions for the chow parameters problem and
  low-weight approximation of halfspaces.
\newblock {\em J. {ACM}}, 61(2):11:1--11:36, 2014.

\bibitem[DDS12a]{DDS12stoc}
C.~Daskalakis, I.~Diakonikolas, and R.A. Servedio.
\newblock {Learning Poisson Binomial Distributions}.
\newblock In {\em STOC}, pages 709--728, 2012.

\bibitem[DDS12b]{DDS12icalp}
A.~De, I.~Diakonikolas, and R.~A. Servedio.
\newblock The inverse shapley value problem.
\newblock In {\em ICALP (1)}, pages 266--277, 2012.

\bibitem[DDS15]{DDS15}
A.~De, I.~Diakonikolas, and R.~Servedio.
\newblock Learning from satisfying assignments.
\newblock In {\em Proceedings of the 26th Annual ACM-SIAM Symposium on Discrete
  Algorithms, {SODA} 2015}, pages 478--497, 2015.

\bibitem[Der65]{Dertouzos:65}
M.~Dertouzos.
\newblock {\em {Threshold Logic: A Synthesis Approach}}.
\newblock MIT Press, Cambridge, MA, 1965.

\bibitem[DHK{\etalchar{+}}10]{DHK+:10}
I.~Diakonikolas, P.~Harsha, A.~Klivans, R.~Meka, P.~Raghavendra, R.~A.
  Servedio, and L.~Y. Tan.
\newblock Bounding the average sensitivity and noise sensitivity of polynomial
  threshold functions.
\newblock In {\em STOC}, pages 533--542, 2010.

\bibitem[DKK{\etalchar{+}}16]{DKKLMS16}
I.~Diakonikolas, G.~Kamath, D.~M. Kane, J.~Li, A.~Moitra, and A.~Stewart.
\newblock Robust estimators in high dimensions without the computational
  intractability.
\newblock In {\em Proceedings of FOCS'16}, pages 655--664, 2016.

\bibitem[DKK{\etalchar{+}}17a]{DiakonikolasKKL16-icml}
I.~Diakonikolas, G.~Kamath, D.~M. Kane, J.~Li, A.~Moitra, and A.~Stewart.
\newblock Being robust (in high dimensions) can be practical.
\newblock {\em CoRR}, abs/1703.00893, 2017.

\bibitem[DKK{\etalchar{+}}17b]{DKKLMS17}
I.~Diakonikolas, G.~Kamath, D.~M. Kane, J.~Li, A.~Moitra, and A.~Stewart.
\newblock Robustly learning a gaussian: Getting optimal error, efficiently.
\newblock {\em CoRR}, abs/1704.03866, 2017.

\bibitem[DKS16]{DiakonikolasKS16c}
I.~Diakonikolas, D.~M. Kane, and A.~Stewart.
\newblock Statistical query lower bounds for robust estimation of
  high-dimensional gaussians and gaussian mixtures.
\newblock {\em CoRR}, abs/1611.03473, 2016.

\bibitem[DL01]{DL:01}
L.~Devroye and G.~Lugosi.
\newblock {\em Combinatorial methods in density estimation}.
\newblock Springer Series in Statistics, Springer, 2001.

\bibitem[DLS14]{DanielyLS14}
A.~Daniely, N.~Linial, and S.~S.{-}Shwartz.
\newblock From average case complexity to improper learning complexity.
\newblock In {\em Symposium on Theory of Computing, {STOC} 2014}, pages
  441--448, 2014.

\bibitem[DRST14]{DRST14}
I.~Diakonikolas, P.~Raghavendra, R.~A. Servedio, and L.~Y. Tan.
\newblock Average sensitivity and noise sensitivity of polynomial threshold
  functions.
\newblock {\em {SIAM} J. Comput.}, 43(1):231--253, 2014.

\bibitem[Hau92]{Haussler:92}
D.~Haussler.
\newblock {Decision theoretic generalizations of the PAC model for neural net
  and other learning applications}.
\newblock {\em Information and Computation}, 100:78--150, 1992.

\bibitem[HKM14]{HKM14}
P.~Harsha, A.~R. Klivans, and R.~Meka.
\newblock Bounding the sensitivity of polynomial threshold functions.
\newblock {\em Theory of Computing}, 10:1--26, 2014.

\bibitem[Kan11]{Kane11}
D.~M. Kane.
\newblock The gaussian surface area and noise sensitivity of degree-\emph{d}
  polynomial threshold functions.
\newblock {\em Computational Complexity}, 20(2):389--412, 2011.

\bibitem[Kan14a]{Kane14}
D.~M. Kane.
\newblock The average sensitivity of an intersection of half spaces.
\newblock In {\em Symposium on Theory of Computing, {STOC} 2014}, pages
  437--440, 2014.

\bibitem[Kan14b]{Kane14b}
D.~M. Kane.
\newblock The correct exponent for the gotsman-linial conjecture.
\newblock {\em Computational Complexity}, 23(2):151--175, 2014.

\bibitem[KKMS08]{KKMS:08}
A.~Kalai, A.~Klivans, Y.~Mansour, and R.~Servedio.
\newblock Agnostically learning halfspaces.
\newblock {\em SIAM Journal on Computing}, 37(6):1777--1805, 2008.

\bibitem[KL93]{KearnsLi:93}
M.~Kearns and M.~Li.
\newblock Learning in the presence of malicious errors.
\newblock {\em SIAM Journal on Computing}, 22(4):807--837, 1993.

\bibitem[KLS09]{KLS:09jmlr}
A.~Klivans, P.~Long, and R.~Servedio.
\newblock {Learning Halfspaces with Malicious Noise}.
\newblock {\em Journal of Machine Learning Research}, 10:2715--2740, 2009.

\bibitem[KLT09]{KlivansLT09}
A.~R. Klivans, P.~M. Long, and A.~K. Tang.
\newblock Baum's algorithm learns intersections of halfspaces with respect to
  log-concave distributions.
\newblock In {\em 13th International Workshop, {RANDOM} 2009}, pages 588--600,
  2009.

\bibitem[KOS08]{KOS:08}
A.~Klivans, R.~O'Donnell, and R.~Servedio.
\newblock Learning geometric concepts via {G}aussian surface area.
\newblock In {\em Proc.\ 49th IEEE Symposium on Foundations of Computer Science
  (FOCS)}, pages 541--550, 2008.

\bibitem[KSS94]{KSS:94}
M.~Kearns, R.~Schapire, and L.~Sellie.
\newblock {Toward Efficient Agnostic Learning}.
\newblock {\em Machine Learning}, 17(2/3):115--141, 1994.

\bibitem[MP68]{MinskyPapert:68}
M.~Minsky and S.~Papert.
\newblock {\em {P}erceptrons: an introduction to computational geometry}.
\newblock MIT Press, Cambridge, MA, 1968.

\bibitem[MT94]{MT:94}
W.~Maass and G.~Turan.
\newblock How fast can a threshold gate learn?
\newblock In S.~Hanson, G.~Drastal, and R.~Rivest, editors, {\em Computational
  Learning Theory and Natural Learning Systems}, pages 381--414. MIT Press,
  1994.

\bibitem[Mur71]{Muroga:71}
S.~Muroga.
\newblock {\em Threshold logic and its applications}.
\newblock Wiley-Interscience, New York, 1971.

\bibitem[TTV08]{TTV:09}
L.~Trevisan, M.~Tulsiani, and S.~Vadhan.
\newblock {Regularity, Boosting and Efficiently Simulating every High Entropy
  Distribution }.
\newblock Technical Report 103, Electronic Colloquium in Computational
  Complexity, 2008.
\newblock {Conference version in Proceedings of CCC '09}.

\bibitem[Val84]{val84}
L.~G. Valiant.
\newblock A theory of the learnable.
\newblock In {\em Proc.\ 16th Annual ACM Symposium on Theory of Computing
  (STOC)}, pages 436--445. ACM Press, 1984.

\bibitem[Val85]{Valiant:85}
L.~Valiant.
\newblock Learning disjunctions of conjunctions.
\newblock In {\em Proceedings of the Ninth International Joint Conference on
  Artificial Intelligence}, pages 560--566, 1985.

\bibitem[Vem10a]{Vempala10a}
S.~Vempala.
\newblock Learning convex concepts from gaussian distributions with {PCA}.
\newblock In {\em 51th Annual {IEEE} Symposium on Foundations of Computer
  Science, {FOCS}}, pages 124--130, 2010.

\bibitem[Vem10b]{Vempala10}
S.~Vempala.
\newblock A random-sampling-based algorithm for learning intersections of
  halfspaces.
\newblock {\em J. {ACM}}, 57(6):32:1--32:14, 2010.

\end{thebibliography}

\end{document}